%% file: main.tex
\documentclass[lettersize,journal]{IEEEtran}
\usepackage{amsmath,amsfonts}
\usepackage{algorithmic}
\usepackage{algorithm}
\usepackage{array}
\usepackage[caption=false,font=normalsize,labelfont=sf,textfont=sf]{subfig}
\usepackage{textcomp}
\usepackage{stfloats}
\usepackage{url}
\usepackage{verbatim}
\usepackage{graphicx}
\usepackage{cite}
\usepackage{xcolor}
\usepackage{multirow}
\usepackage{booktabs}
\usepackage{tabularx}
\usepackage{adjustbox}
\usepackage{hyperref}
\usepackage{amsthm}
\usepackage{float}
\usepackage{balance}
\usepackage{xcolor}

\begin{document}

\title{Gated Adaptation for Continual Learning in Human Activity Recognition}

\author{Reza Rahimi Azghan$^{\star}$, Gautham Krishna Gudur$^{\dagger}$, Mohit Malu$^{\star}$ \\Edison Thomaz$^{\dagger}$, Giulia Pedrielli$^{\star}$, Pavan Turaga$^{\star}$, Hassan Ghasemzadeh$^{\star}$
\thanks{$^{\star}$Arizona State University, Phoenix, AZ, USA,\vfill$^{\dagger}$ The University of Texas at Austin, Austin, TX, USA}%
\thanks{Email: \{rrahimia, hassan.ghasemzadeh, pavan.turaga, giulia.pedrielli, mmalu\}@asu.edu, \{ethomaz, gauthamkrishna\}@utexas.edu }%
\thanks{This work was supported in part by the National Science Foundation under
grant IIS-2402650. Any opinions, findings, conclusions, or recommendations expressed
in this material are those of the authors and do not necessarily reflect the
views of the funding organization.}
\thanks{Copyright (c) 20xx IEEE. Personal use of this material is permitted. However, permission to use this material for any other purposes must be obtained from the IEEE by sending a request to \href{pubs-permissions@ieee.org}{pubs-permissions@ieee.org}.}
}

% The paper headers
% \markboth{Journal of \LaTeX\ Class Files,~Vol.~14, No.~8, August~2021}%
% {Shell \MakeLowercase{\textit{et al.}}: A Sample Article Using IEEEtran.cls for IEEE Journals}

% \IEEEpubid{0000--0000/00\$00.00~\copyright~2021 IEEE}
% Remember, if you use this you must call \IEEEpubidadjcol in the second
% column for its text to clear the IEEEpubid mark.

\maketitle

\input{Sections/00_abstract}
\input{Sections/01_introduction}
\input{Sections/02_related_work}
\input{Sections/03_proposed_method}

\input{Sections/04_theoretical_insights}

\input{Sections/05_experimntal_setup}
\input{Sections/06_experimental_results}
\input{Sections/07_conclusion}

\bibliographystyle{IEEEtran}
\bibliography{refs}

\end{document}

%% file: Sections/00_abstract.tex
\begin{abstract}
Wearable sensors in Internet of Things (IoT) ecosystems increasingly support applications such as remote health monitoring, elderly care, and smart home automation, all of which rely on robust human activity recognition (HAR). Continual learning systems must balance plasticity (the ability to learn new tasks) with stability (the retention of previously acquired knowledge). However, AI models often exhibit catastrophic forgetting, where learning new tasks degrades performance on earlier ones. This challenge is particularly acute in domain-incremental settings such as HAR, where on-device models must adapt to new subjects with distinct movement characteristics while maintaining accuracy on previously seen subjects without transmitting sensitive data to the cloud. In this work, we propose a parameter-efficient continual learning framework based on channel-wise gated modulation applied to frozen pretrained representations. Our key insight is that adaptation should operate through feature \emph{selection} rather than feature \emph{generation}: by restricting learned transformations to diagonal scaling of existing features, we preserve the geometric structure of pretrained representations while enabling subject-specific modulation. We provide a theoretical analysis showing that gating implements a bounded diagonal operator that limits representational drift compared to unconstrained linear transformations. Empirically, we demonstrate that freezing the backbone substantially reduces forgetting and that lightweight gates restore the adaptation capacity lost from freezing, achieving both stability and plasticity simultaneously. On the PAMAP2 dataset with 8 sequential subjects, our approach reduces forgetting from 39.7\% (trainable backbone) to 16.2\% and improves final accuracy from 56.7\% to 77.7\%, while training less than 2\% of model parameters. Our method matches or exceeds standard continual learning baselines without requiring replay buffers or task-specific regularization, confirming that structured diagonal operators provide an effective and efficient mechanism for continual learning under distribution shift.\looseness=-1
\end{abstract}


%% file: Sections/01_introduction.tex
\section{Introduction}

Adapting to a dynamic environment is one of the fundamental attributes of intelligent systems. Human beings, as the most capable example of intelligence in the natural world, have developed this ability through years of evolution. They are able to learn new concepts and acquire new knowledge without forgetting what they have previously learned. In artificial intelligence, this capability is studied under the framework of continual learning~\cite{PARISI201954}, which refers to models that incrementally learn from a continuous stream of data and tasks over time, adapt to changing environments without retraining from scratch, and, critically, retain knowledge acquired from earlier experiences. To support such behavior, continual learning systems must address the well-known stability–plasticity dilemma~\cite{mccloskey1989catastrophic}, where models need to remain sufficiently plastic to learn new tasks while also being stable enough to preserve previously learned information.\looseness=-1

In practice, learning systems often favor plasticity over stability, which leads to a phenomenon known as catastrophic forgetting~\cite{FRENCH1999128}, where performance on previously learned tasks deteriorates as new tasks are introduced. This issue becomes more severe in real-world deployments, where models are exposed to large volumes of data originating from multiple, potentially complex and evolving distributions. In these settings, data is commonly organized into tasks, where each task consists of a set of samples generated under consistent conditions. More precisely, a task consists of data points drawn from a shared generative distribution. While different tasks may vary in their data distributions, samples within a task are assumed to exhibit similar statistical properties. Continual learning systems are therefore required to accommodate shifts between such distributions over time while maintaining performance on previously encountered tasks.\looseness=-1

The proliferation of wearable sensors in Internet of Things (IoT) ecosystems has enabled a wide range of applications that depend on robust human activity recognition (HAR), including remote health monitoring~\cite{qi2017advanced, hashemi2025ultra}, elderly care and fall detection~\cite{wang2019survey}, smart home automation~\cite{yao2017deepsense, mohammadibalini2026disparities}, and fitness tracking. In these deployments, HAR models typically execute on resource-constrained edge devices such as smartwatches, fitness bands, or smartphones, where they must operate under strict memory, compute, and energy budgets~\cite{lane2015deepear}. A critical requirement for such systems is the ability to personalize to new users over time: as a wearable device is adopted by different household members or as a telehealth system onboards new patients, the model must adapt to each individual's movement characteristics without forgetting previously learned users.\looseness=-1

Sensor-based HAR~\cite{chen2021deep} provides a natural testbed for studying this challenge, as individual subjects define task boundaries~\cite{sah2022continual}. Each person exhibits distinct movement patterns, physiological characteristics, and sensor placement variations. While activity labels remain constant (e.g., walking, sitting, running), the manner in which these activities are performed varies substantially across subjects, inducing subject-specific data distributions. Importantly, transmitting raw sensor data to cloud servers for centralized retraining raises significant privacy concerns, as movement patterns can reveal sensitive health information and behavioral profiles~\cite{lane2010survey}. This motivates \emph{on-device continual learning}, where models adapt locally to new subjects without requiring data transmission. Figure~\ref{fig:forgetting_intro} illustrates the severity of this challenge on the PAMAP2 dataset~\cite{pamap2_physical_activity_monitoring_231}, where a multilayer perceptron model suffers catastrophic forgetting: after training on only four subjects, the accuracy on the first subject drops from 85\% to 40\%, a decline of 45 percentage points.\looseness=-1

\begin{figure}
    \centering
    \includegraphics[width=\linewidth]{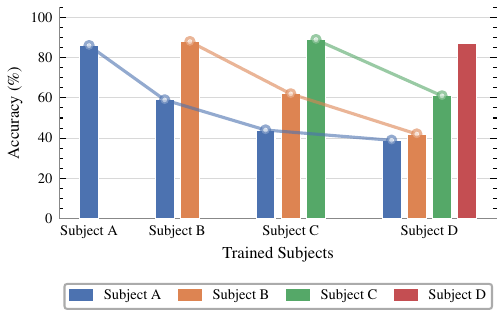}
    \caption{Catastrophic forgetting in subject-incremental HAR: accuracy on Subject 1 drops from 85\% to 40\% after training on just three additional subjects (PAMAP2 dataset).}
    \label{fig:forgetting_intro}
\end{figure}

To address catastrophic forgetting, three main paradigms have emerged in the continual learning literature. \emph{Regularization-based methods}~\cite{kirkpatrick2017ewc, zenke2017continual} constrain parameter updates by penalizing changes to weights deemed important for previous tasks, typically through Fisher information or importance estimates. \emph{Replay-based methods}~\cite{rebuffi2017icarl, buzzega2020darkexperiencegeneralcontinual} store or generate samples from prior tasks to maintain performance through periodic rehearsal. \emph{Architectural methods}~\cite{rusu2016progressive, serrà2018overcomingcatastrophicforgettinghard} expand model capacity or selectively activate subnetworks for different tasks. While these approaches have demonstrated effectiveness in controlled settings, they face practical limitations for IoT deployment: regularization methods can be overly conservative and require storing per-task importance weights, replay methods raise privacy concerns and demand persistent storage of sensitive sensor data, and architectural expansion increases model size beyond the memory constraints of edge devices.\looseness=-1

More recently, \emph{frozen backbone} strategies have gained traction~\cite{wang2022l2p, wang2022dualprompt}, where a pretrained feature extractor remains fixed while only task-specific components are trained. This paradigm offers computational efficiency and inherent stability, as the shared representation cannot be corrupted by new task updates. However, a critical question remains: how should adaptation occur on top of a frozen backbone? Common approaches include training only the final classifier or adding trainable adapter layers~\cite{houlsby2019adapters}. The former provides minimal adaptation capacity, often insufficient for significant distribution shifts, while the latter introduces dense transformations that can generate entirely new feature spaces, undermining the stability benefits of the frozen backbone. This motivates our focus on structured, low-capacity adaptation mechanisms that balance expressiveness with stability.\looseness=-1

Within the context of parameter-efficient fine-tuning, we propose the use of gated modulation mechanisms to selectively adapt intermediate representations of a neural network. Specifically, we introduce lightweight channel-wise gates that operate on the outputs of intermediate layers and modulate feature activations based on the input. The proposed gating mechanism is inspired by the design principles of squeeze-and-excitation networks~\cite{hu2018senet}, but is employed here as a means of continual adaptation rather than static feature recalibration. By placing these gates on top of a frozen pretrained backbone, the model is able to adjust the relative importance of learned features while preserving the stability of the underlying representation.\looseness=-1

Our analysis demonstrates that this form of structured, low-capacity adaptation reduces catastrophic forgetting compared to fully fine-tuning all model parameters, as updates are constrained to bounded, task-adaptive gates and a shared classifier. Moreover, in contrast to approaches that introduce additional trainable layers on top of frozen backbones, gated modulation provides a more controlled form of adaptation by restricting changes to channel-wise reweighting rather than arbitrary feature transformations. We support these claims through both theoretical analysis and extensive empirical evaluation on multiple sensor-based human activity recognition datasets.\looseness=-1

In summary, the contributions of this paper are as follows:\looseness=-1
\begin{itemize}
\item We propose a parameter-efficient continual learning framework that employs channel-wise gated modulation to adapt frozen pretrained neural representations under subject-induced distribution shifts. The approach requires less than 2\% of trainable parameters compared to full fine-tuning, making it suitable for deployment on resource-constrained IoT edge devices.\looseness=-1

\item We provide a theoretical analysis demonstrating that gated adaptation implements a diagonal operator. This mechanism bounds representational drift through structured channel-wise reweighting, limiting interference with previously learned tasks.\looseness=-1

\item We conduct extensive empirical evaluations across three sensor-based HAR benchmarks (PAMAP2, UCI-HAR, DSA) under domain-incremental settings with up to 30 sequential subjects, demonstrating consistent improvements in the stability-plasticity tradeoff. Our approach substantially outperforms both minimal adaptation (frozen backbone with classifier only) and high-capacity adaptation (stacked trainable layers) across all datasets, validating our theoretical predictions.\looseness=-1
\end{itemize}

%% file: Sections/02_related_work.tex
\section{Related Work}

Our work sits at the intersection of continual learning, parameter-efficient fine-tuning, attention mechanisms, and human activity recognition. We review each area and identify the specific gap our approach addresses.

\subsection{Continual Learning and Catastrophic Forgetting}

Continual learning addresses the challenge of training neural networks on sequential tasks without forgetting previously acquired knowledge, a phenomenon known as catastrophic forgetting \cite{kirkpatrick2017ewc, mccloskey1989catastrophic}. Three main paradigms have emerged to address this challenge. \textit{Regularization-based methods} constrain parameter updates to preserve important weights: Elastic Weight Consolidation (EWC) \cite{kirkpatrick2017ewc} penalizes changes to parameters deemed important for prior tasks via the Fisher information matrix, while Learning without Forgetting (LwF) \cite{li2017lwf} uses knowledge distillation to maintain output consistency. \textit{Replay-based methods} store or generate exemplars from previous tasks \cite{rebuffi2017icarl, shin2017dgr}, though this raises privacy and memory concerns. \textit{Architecture-based methods} allocate task-specific parameters \cite{rusu2016progressive, mallya2018packnet}, but often incur growing model complexity. Recent comprehensive surveys \cite{wang2024clsurvey, zhou2024cilsurvey, vandeven2024clbook} highlight that while significant progress has been made, existing approaches involve trade-offs between stability (retaining old knowledge) and plasticity (learning new tasks). Our work contributes to the architecture-based paradigm by introducing lightweight, bounded gating mechanisms that modulate frozen representations.\looseness=-1

\subsection{Parameter-Efficient Fine-Tuning for Continual Learning}

The rise of large pretrained models has motivated parameter-efficient fine-tuning (PEFT) methods that adapt models by updating only a small subset of parameters while keeping the backbone frozen \cite{houlsby2019adapters, hu2022lora}. This frozen backbone paradigm provides inherent stability against forgetting, as the core representations remain unchanged. Recent work has explicitly connected PEFT with continual learning: Learning to Prompt (L2P) \cite{wang2022l2p} maintains a learnable prompt pool to dynamically instruct a frozen Vision Transformer across sequential tasks without rehearsal. DualPrompt \cite{wang2022dualprompt} extends this with task-invariant and task-specific prompt components, while CODA-Prompt \cite{smith2023codaprompt} introduces end-to-end attention-based prompt assembly. A recent survey on pretrained model-based continual learning \cite{zhou2024ptmclsurvey} systematically analyzes these approaches, noting their success in rehearsal-free settings. However, critical analysis \cite{thede2024reflecting} reveals that simple baselines with frozen backbones can be surprisingly competitive, questioning the necessity of complex prompt mechanisms.

\subsection{Gating and Channel Attention Mechanisms}

Squeeze-and-Excitation Networks (SE-Nets) \cite{hu2018senet} introduced channel-wise recalibration, learning to weight feature channels based on global context via a squeeze-excitation block. This mechanism has been extended to spatial attention in CBAM \cite{woo2018cbam} and applied to feature conditioning in FiLM \cite{perez2018film}. In the continual learning context, Masse et al.~\cite{masse2018gating} demonstrated that context-dependent gating inspired by neuroscience can alleviate forgetting by selectively activating subnetworks. For human activity recognition specifically, SE-blocks have been integrated into deep residual networks for wearable sensor data \cite{mekruksavanich2022senethar}, showing improved recognition accuracy through channel attention. Our work is inspired by SE-Nets but applies gating in a fundamentally different context: rather than learning static attention weights, we train gates that modulate a frozen pretrained backbone for domain-incremental adaptation. This provides bounded feature drift while maintaining representational stability.

\subsection{Continual Learning for Human Activity Recognition}

Human activity recognition (HAR) from wearable sensors presents a natural testbed for continual learning, as deployed systems must adapt to new users, devices, and environments over time \cite{chen2022clhar, jha2021harbenchmark, ExtendedTemporalTasks}. Subject-wise variation, where different individuals perform the same activity with distinct motion patterns, induces domain shifts that cause catastrophic forgetting when models are transferred across users. Jha et al.~\cite{jha2021harbenchmark} provided an empirical benchmark evaluating standard continual learning techniques (EWC, LwF, replay) on HAR datasets, finding that replay-based methods generally outperform regularization approaches but at the cost of memory overhead. Recent advances include online continual learning frameworks \cite{schiemer2023oclhar} that handle streaming sensor data, prototype-based approaches \cite{adaimi2022lapnethar, farahmand2025attenglucomultimodaltransformerbasedblood} using experience replay with continual prototype evolution, and self-supervised methods that leverage unlabeled data \cite{bock2024casslehar}. Most recently, CLAD-Net \cite{azghan2025cladnetcontinualactivityrecognition} combines self-supervised representation learning with knowledge distillation for domain-incremental HAR.

In addition to these efforts, on-device continual learning has also been explored for resource-constrained platforms. HARNet \cite{sundaramoorthy2018harnet} proposes a lightweight deep ensemble architecture designed for incremental learning directly on constrained mobile devices. By enabling models to adapt to new activities without retraining from scratch and without relying on cloud resources, HARNet demonstrates that continual adaptation can be achieved efficiently under strict computational and memory budgets. This line of work highlights the growing need for continual learning systems that remain practical in real-world, embedded HAR deployments.

Recent research has also examined lifelong adaptation through metric-based learning. Adaimi and Thomaz \cite{adaimi2022lapnethar} introduced a prototypical-network-based framework that incrementally updates class prototypes to support lifelong sensor-based activity recognition. Their approach reduces catastrophic forgetting by anchoring new representations around stable prototypes while enabling rapid adaptation to additional classes or users. This prototypical perspective complements architecture- and replay-based approaches by emphasizing representation stability during incremental updates.

\subsection{Summary and Research Gap}

While frozen backbone approaches with PEFT have shown promise for continual learning, existing methods rely on additive mechanisms (prompts, adapters) that lack theoretical guarantees on representation drift. Simultaneously, gating mechanisms have proven effective for channel recalibration but have not been systematically applied to frozen backbones in domain-incremental settings. In HAR specifically, continual learning research has focused on applying general-purpose methods without leveraging the channel-wise structure of subject-induced domain shifts. Our work addresses this gap by combining a frozen pretrained backbone with lightweight, channel-wise gates, providing both theoretical bounds on feature drift and empirical improvements in knowledge retention.\looseness=-1

Although our gates share the parameter-efficient, frozen-backbone philosophy of prompt- and adapter-based methods \mbox{\cite{wang2022l2p, wang2022dualprompt, houlsby2019adapters, hu2022lora}}, they differ in a fundamental way. Adapters and prompts are additive: they introduce new parameters that generate feature subspaces outside the pretrained representation, and the resulting drift is unconstrained, since the injected transformation can in principle move features arbitrarily. Our gates are instead multiplicative and diagonal: they reweight existing channels rather than synthesize new ones, which bounds representational drift by construction (Theorem~1) and underlies our stability guarantees. Our approach also differs from masking-based methods such as HAT \mbox{\cite{serrà2018overcomingcatastrophicforgettinghard}} and PackNet \mbox{\cite{mallya2018packnet}}. These partition network capacity by learning hard, typically binary masks that allocate disjoint parameter subsets per task, requiring per-task masks and explicit task identity. In contrast, our gates apply continuous channel-wise scaling to a single shared representation, store one set of parameters regardless of the number of tasks, and operate without task boundaries at test time. In short, prior methods adapt through feature generation or capacity partitioning, whereas our method adapts through bounded feature selection

%% file: Sections/03_proposed_method.tex
\section{Proposed Approach}

\begin{figure*}
    \centering
    \includegraphics[width=\linewidth]{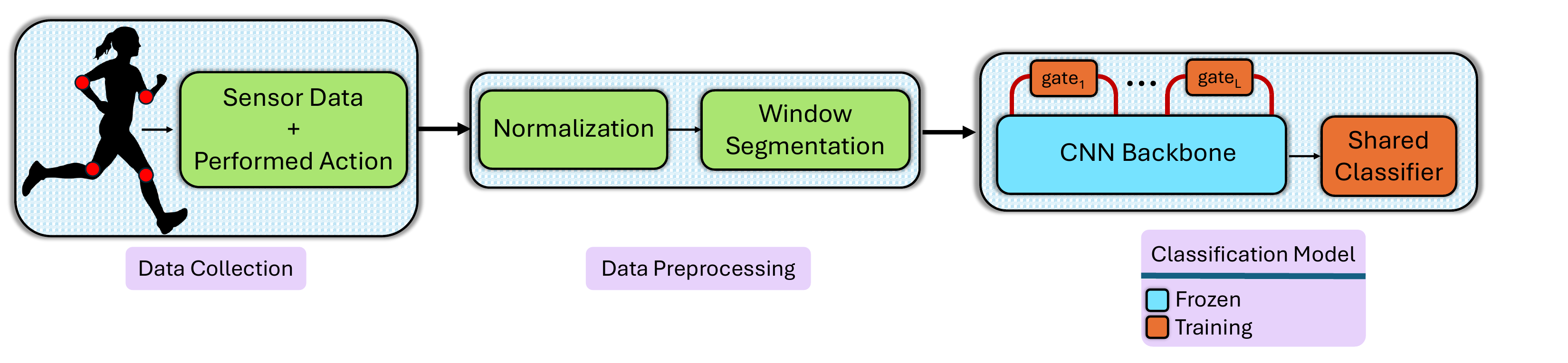}
    \caption{Overview of the proposed continual learning system for HAR, comprising data collection, preprocessing, and a classification model with a frozen pretrained backbone modulated by lightweight trainable gates.}
    \label{fig:system}
\end{figure*}

\subsection{System Overview}
This section presents our proposed system, structured as a pipeline from raw sensor reading to continual model adaptation. As illustrated in Figure \ref{fig:system}, the system comprises three core components.\looseness=-1

During data collection, inertial measurement units (IMUs) are placed on multiple body locations while subjects perform a predefined set of activities. Subjects provide activity labels indicating the action performed at each time instant, enabling supervised training. Next, raw sensor streams from each subject undergo per-subject normalization to account for individual differences in sensor placement and motion intensity. A sliding window segmentation algorithm then partitions the continuous data into fixed-length segments. Each window is represented as a tuple, $(\mathbf{x}, y, t)$, which represents the normalized input data, the activity associated with the input, and the user ID from whom the data was collected.\looseness=-1

The final component is a continual learning model designed to classify segmented windows into activity categories. The central objective is to achieve high recognition accuracy on incoming subjects (plasticity) while preserving performance on previously encountered subjects (stability). To address this stability-plasticity trade-off, we introduce lightweight trainable gates that modulate a pretrained backbone's representations to accommodate subject-specific domain shifts. The architecture and its associated training objectives are detailed in the following subsections.\looseness=-1

\subsection{Problem Formulation}

We formulate continual learning for human activity recognition as a \emph{domain-incremental} learning problem. The learner encounters a sequence of $T$ tasks, where each task $t \in \{1,\dots,T\}$ corresponds to a distinct data distribution representing a new subject. At task $t$, the model receives training data,\looseness=-1
\[
\mathcal{D}_{t}=\{(\mathbf{x}_{i,t}, y_{i,t})\}_{i=1}^{N_t},
\]
where $\mathbf{x}_{i,t}\in \mathbb{R}^{c_0\times d_0}$ is a multivariate time series of length $d_0$ with $c_0$ sensor channels, and $y_{i,t}\in \mathcal{Y}$ is the corresponding activity label from a fixed label set. Each task is drawn from a subject-specific distribution $\mathcal{D}_t \sim \mathbf{P}_{t}(\mathbf{x}, y)$.

Under the domain-incremental assumption, the label distribution remains stationary across tasks,
\[
\mathbf{P}_{t}(y)=\mathbf{P}_{t+1}(y), \quad \forall t\in \{1, \dots, T-1\},
\]
while the marginal distribution over inputs may shift due to subject-specific characteristics,
\[
\mathbf{P}_{t}(\mathbf{x})\neq \mathbf{P}_{t'}(\mathbf{x}), \quad \forall t\neq t'.
\]
This captures the HAR scenario: all subjects perform the same activities (fixed label semantics), but exhibit individual variations in movement patterns, sensor placement, and physiological characteristics (varying input statistics).\looseness=-1

The model is exposed to tasks sequentially and cannot store or revisit data from previous tasks. This reflects memory constraints in embedded and mobile deployment scenarios. The objective is to maintain high classification accuracy on all encountered tasks, including both the current task and all previous tasks, thereby balancing \emph{plasticity} (adaptation to new subjects) with \emph{stability} (retention of knowledge about earlier subjects).\looseness=-1

\subsection{Model Overview}

As illustrated in Figure \ref{fig:system}, our proposed model consists of a frozen pretrained backbone, a set of lightweight channel-wise gating modules inserted after each backbone block, and a shared classifier trained across all tasks.\looseness=-1

\paragraph{Frozen Pretrained Backbone} The backbone comprises $L$ residual convolutional blocks $\{{\phi_\ell}\}_{\ell=1}^L$, each mapping $\phi_{\ell}:\mathbb{R}^{c_{\ell-1} \times d_{\ell-1}} \to \mathbb{R}^{c_{\ell} \times d_{\ell}}$. Given an input sequence $\mathbf{x}$, the output of block $\ell$ is,\looseness=-1

$$U_{\ell}=\phi_{\ell}(H_{\ell-1}), \hspace{1em}H_0=\mathbf{x}$$

These blocks are pretrained on a source HAR dataset and remain fixed throughout continual learning. Freezing the backbone ensures representational stability, as the shared feature extractor does not undergo updates that could degrade performance on earlier subjects. Moreover, this strategy reduces the number of trainable parameters by orders of magnitude, yielding substantial computational efficiency and enabling deployment on lightweight devices.\looseness=-1

\begin{figure*}
    \centering
    \includegraphics[width=\linewidth]{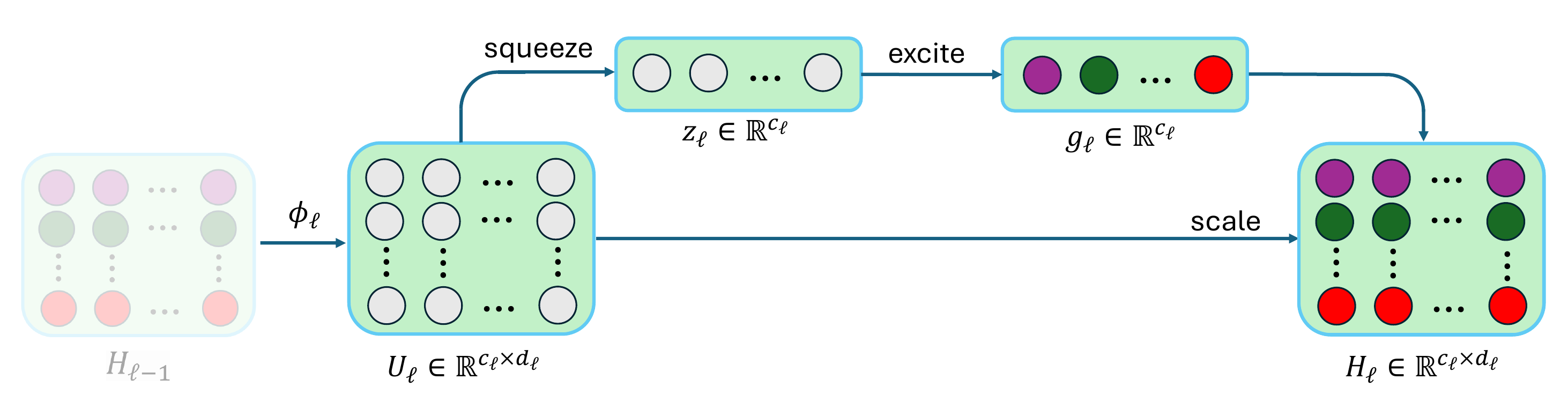}
    \caption{Channel-wise gating mechanism. Given an intermediate feature map $U_\ell$, global average pooling (squeeze) produces a channel descriptor $z_\ell$, which is transformed through bottleneck layers (excitation) to produce gate values $g_\ell \in (0,1)^{c_\ell}$. The gates then scale each channel independently, preserving feature directions while modulating magnitudes.}
    \label{fig:gate}
\end{figure*}

\paragraph{Channel-Wise Gates}  At the output of each intermediate layer in the backbone, we introduce lightweight gating modules that adaptively reweight channel activations. The gating mechanism, illustrated in Figure \ref{fig:gate}, is inspired by Squeeze-and-Excitation Networks (SENets)\cite{hu2018senet} but repurposed for continual adaptation. For a feature map $U_\ell \in \mathbb{R}^{c_\ell \times d_\ell}$ at an intermediate layer, the gate starts by computing a global descriptor by aggregating temporal information.\looseness=-1

$$
z_\ell = \dfrac{1}{d_\ell}\sum_{j=1}^{d_\ell}U_{\ell, :, j}
$$
where $z_\ell \in \mathbb{R}^{c_\ell}$ denotes the channel-wise mean activation aggregated over all temporal positions. In the original SENet formulation, this operation corresponds to the squeeze stage. 

The resulting descriptor $z_\ell$ is then passed through an excitation mapping parameterized by learnable weights. In our implementation, the excitation function is implemented using two fully connected layers that transform the compressed representation $z_\ell$ into channel-wise gating coefficients $g_\ell$ via a sigmoid nonlinearity, given by,
$$
    g_\ell = \sigma(\mathbf{W}_{2,\ell}  \cdot \text{ReLU}(\mathbf{W}_{1,\ell} \cdot z_{\ell}))
$$
where $\mathbf{W}_{1,\ell} \in \mathbb{R}^{c_{\ell}/r \times c_\ell}$ and $\mathbf{W}_{2,\ell}\in \mathbb{R}^{c_{\ell} \times c_{\ell}/r}$ are learnable bottleneck layers with a reduction ratio $r$, and $\sigma$ is the sigmoid activation that constrains the gating vector to $g_\ell \in (0,1)^{c_\ell}$.

The final step of the gating module applies the learned channel-wise coefficients to the intermediate feature representation. Specifically, the gating vector $g_\ell$ is used to reweight the channels of $U_\ell$ through a diagonal scaling operation,

$$
    H_\ell = \text{D}(g_\ell) \cdot U_{\ell}, \label{eq:gate_apply}
$$
where  
\[
\text{D}(g_\ell) = \begin{bmatrix}
g_{\ell, 1} & 0 & \cdots & 0 \\
0 & g_{\ell, 2} & \cdots & 0 \\
\vdots & \vdots & \ddots & \vdots \\
0 & 0 & \cdots & g_{\ell, c_\ell}
\end{bmatrix}
\]

This diagonal operation restricts adaptation to channel-wise \emph{magnitude scaling} while preserving the \emph{directional structure} of the pretrained features. This bounded form of modulation limits representational drift and enables stable continual learning.\looseness=-1

\paragraph{Shared Classifier} Finally, we employ a single-layer linear classifier shared across all tasks and continuously updated as new subjects are introduced. First, a temporal pooling operator is applied to the final gated feature map $H_L$ to obtain a fixed-dimensional representation,\looseness=-1

$$h = \dfrac{1}{d_L}\sum_{j=1}^{d_L} H_{L,:,j}$$

followed by a linear operation,

$$\hat{y} = \mathbf{W}_{\text{cl}} h$$

where $\mathbf{W}_{\text{cl}} \in \mathbb{R}^{|\mathcal{Y}| \times c_L}$, and is trained continually as new tasks are encountered. This design forces the model to learn a common decision boundary in the gated feature space rather than task-specific output mappings.\looseness=-1

Algorithm~\ref{alg:training} summarizes the training procedure. Here, only the gate parameters $\mathbf{W}_{1,\ell}, \mathbf{W}_{2, \ell}$ and the classifier $\mathbf{W}_{\text{cl}}$ are updated for each task. The backbone blocks remain frozen throughout. This dramatically reduces the number of trainable parameters per task (typically less than 2\% of total model parameters) and enforces stability by preventing changes to the shared feature representation. Furthermore, the algorithm does not store data from previous tasks and does not require explicit task boundary signals at test time. The model must infer subject identity implicitly from input statistics via the gating mechanism.\looseness=-1

\begin{algorithm}[h]
\caption{Continual Training with Interleaved Channel-Wise Gates}
\label{alg:training}
\begin{algorithmic}[1]
\REQUIRE Frozen pretrained backbone blocks $\{\phi_\ell\}_{\ell=1}^{L}$, gate modules with parameters $\{\mathbf{W}_{1,\ell}\, \mathbf{W}_{2,\ell}\}_{\ell=1}^{L}$, classifier with parameter $\mathbf{W}_{\text{cl}}$
\REQUIRE data stream $(\mathbf{x}_t, y_t)$, epochs $E$, optimizer $\mathcal{O}$
\STATE Initialize all parameters  $\mathbf{\Theta} = \{\{\mathbf{W}_{1,\ell}\, \mathbf{W}_{2,\ell}\}_{\ell=1}^{L}, \mathbf{W}_{\text{cl}}\}$; freeze $\{\phi_\ell\}_{\ell=1}^{L}$
\FOR{$e=1,\dots,E$}
    \FOR{each minibatch $\mathcal{B}$ drawn from $(\mathbf{x}_t,y_t)$}
        \FOR{each $(\mathbf{x},y)\in\mathcal{B}$}
            \STATE $H_0 \gets \mathbf{x}$
            \FOR{$\ell=1,\dots,L$}
                \STATE $U_\ell \gets \phi_\ell(H_{\ell-1})$
                \STATE $z_\ell$ = Pool($U_\ell$)
                \STATE    $g_\ell \gets \sigma\!\left(\mathbf{W}_{2,\ell} \cdot \text{ReLU}(\mathbf{W}_{1,\ell} \cdot z_\ell)\right)
$
                \STATE $H_\ell \gets \text{D}(g_\ell)\cdot U_\ell$
            \ENDFOR
            \STATE $h \gets \text{Pool}(H_L)$
            \STATE $\hat{y} \gets \mathbf{W}_{\text{cl}}\cdot h$
            \STATE $\mathcal{L} \gets \text{CrossEntropy}(\hat{y}, y) + \lambda \|\mathbf{\Theta}\|_2^2$
        \ENDFOR
        \STATE Compute gradients $\nabla_{\mathbf{\Theta}}\mathcal{L}$
        \STATE Update model parameters $\mathbf{\Theta}\gets \mathcal{O}(\mathbf{\Theta}, \nabla_{\mathbf{\Theta}}\mathcal{L})$
        
    \ENDFOR
\ENDFOR
\end{algorithmic}
\end{algorithm}

%% file: Sections/04_theoretical_insights.tex
\newtheorem{theorem}{Theorem}
\newtheorem{corollary}{Corollary}
\newtheorem{proposition}[theorem]{Proposition}
\newtheorem{remark}{Remark}
\newtheorem{assumption}{Assumption}

\section{Stability–Plasticity Trade-offs in Gated Continual Learning}

The theoretical analysis clarifies why channel-wise gating yields improved stability in domain-incremental HAR. By restricting adaptation to bounded, diagonal modulation of frozen backbone features, updates induced by new subjects result in limited functional drift on previously seen data. This contrasts with full fine-tuning, classifier-only, or stacked-layer adaptations, where accommodating subject-specific shifts often requires larger or less structured parameter updates, increasing interference across tasks. In the following, we establish formal guarantees on stability, characterize the expressiveness of diagonal gating, and analyze why this mechanism achieves a favorable stability–plasticity trade-off.\looseness=-1

%==============================================================================
\subsection{Stability Guarantees}
%==============================================================================

We begin by analyzing how gated adaptation affects the internal representations learned for previously seen subjects. Because the backbone is frozen, all representational changes induced by learning a new task arise exclusively from updates to the channel-wise gates. This allows us to characterize forgetting in terms of how much the gated feature maps drift on inputs from earlier tasks.\looseness=-1

\begin{theorem}[Bounded Feature Drift under Diagonal Gating]
\label{thm:feature_drift}
Let $U(\mathbf{x}) \in \mathbb{R}^{C \times d}$ denote the ungated feature map produced by a frozen backbone for an input $\mathbf{x}$. Let $g(\mathbf{x}), g'(\mathbf{x}) \in (0,1)^C$ be the channel-wise gate vectors before and after learning a new task, and define the corresponding gated representations,\looseness=-1
\[
H(\mathbf{x}) = D(g(\mathbf{x}))\,U(\mathbf{x}), 
\qquad
H'(\mathbf{x}) = D(g'(\mathbf{x}))\,U(\mathbf{x}),
\]
where $D(\cdot)$ denotes a diagonal matrix. Then the feature drift satisfies,\looseness=-1
\[
\|H'(\mathbf{x}) - H(\mathbf{x})\|_F
\;\le\;
\delta(\mathbf{x})\,\|U(\mathbf{x})\|_F,
\]
where $\delta(\mathbf{x}) = \|g'(\mathbf{x}) - g(\mathbf{x})\|_\infty < 1$.
\end{theorem}

\begin{proof}
Because the backbone parameters are frozen, the ungated feature map $U(\mathbf{x})$ remains unchanged before and after learning a new task. The difference between the gated representations can therefore be written as,\looseness=-1
\begin{align*}
H'(\mathbf{x}) - H(\mathbf{x})
&= \big(D(g'(\mathbf{x})) - D(g(\mathbf{x}))\big)\,U(\mathbf{x}) \\
&= D\big(g'(\mathbf{x}) - g(\mathbf{x})\big)\,U(\mathbf{x}). 
\end{align*}
Applying submultiplicativity of matrix norms and using the fact that the spectral norm of a diagonal matrix equals the infinity norm of its diagonal entries, we obtain,\looseness=-1
\begin{align*}
\|H'(\mathbf{x}) - H(\mathbf{x})\|_F
&\le \|D(g'(\mathbf{x}) - g(\mathbf{x}))\|_2\,\|U(\mathbf{x})\|_F \\
&= \|g'(\mathbf{x}) - g(\mathbf{x})\|_\infty\,\|U(\mathbf{x})\|_F.
\end{align*}
Finally, because each gate component is produced by a sigmoid nonlinearity, we have $g_i(\mathbf{x}) \in (0,1)$ for all channels $i$, implying $\delta(\mathbf{x}) < 1$.\looseness=-1
\end{proof}

\begin{remark}[Contrast with Unconstrained Adaptation]
\label{rem:unconstrained}
If the backbone were trainable, the feature drift would instead be bounded by $\|U'(\mathbf{x}) - U(\mathbf{x})\|_F$, which can be arbitrarily large since $U'$ depends on all backbone parameters. By freezing the backbone and restricting updates to gates, we transform an unbounded representation shift into a multiplicatively bounded one. This is the fundamental mechanism underlying the stability of our approach.\looseness=-1
\end{remark}

Having established that gated adaptation induces a bounded change in the intermediate feature maps, we now examine how this drift propagates through temporal pooling and the shared linear classifier. Ultimately, forgetting occurs when the logits associated with an input from a previously seen subject shift enough to alter the predicted label.\looseness=-1

\begin{theorem}[Bounded Logit Drift]
\label{thm:logit_drift}
Let $U(\mathbf{x}) \in \mathbb{R}^{C \times d}$ denote the ungated backbone feature map of temporal length $d$. Let
\[
H(\mathbf{x}) = D(g(\mathbf{x}))\,U(\mathbf{x}), 
\qquad
H'(\mathbf{x}) = D(g'(\mathbf{x}))\,U(\mathbf{x}),
\]
be the gated representations before and after learning a new task. Define the pooled vectors
\[
h(\mathbf{x}) = \frac{1}{d}\, H(\mathbf{x})\mathbf{1}, 
\qquad
h'(\mathbf{x}) = \frac{1}{d}\, H'(\mathbf{x})\mathbf{1},
\]
and let the classifier change from $W$ to $W'$, with predictions
\[
\hat{y}(\mathbf{x}) = W h(\mathbf{x}), 
\qquad
\hat{y}'(\mathbf{x}) = W' h'(\mathbf{x}).
\]
Then the logit drift decomposes as
\begin{align}
\|\hat{y}'(\mathbf{x}) - \hat{y}(\mathbf{x})\|_2
&\le 
\overbrace{\|W'\|_2 \cdot \frac{\delta(\mathbf{x})}{\sqrt{d}}\,\|U(\mathbf{x})\|_F}^{\text{gate-induced drift}}
\nonumber\\& +
\underbrace{\|W' - W\|_2\,\|h(\mathbf{x})\|_2}_{\text{classifier-induced drift}},
\label{eq:logit_decomposition}
\end{align}
where $\delta(\mathbf{x}) = \|g'(\mathbf{x}) - g(\mathbf{x})\|_\infty < 1$.
\end{theorem}

\begin{proof}
From Theorem~\ref{thm:feature_drift}, the gated feature drift satisfies
\begin{equation}
\|H'(\mathbf{x}) - H(\mathbf{x})\|_F \le \delta(\mathbf{x})\,\|U(\mathbf{x})\|_F. 
\label{eq:featdrift}
\end{equation}
Temporal average pooling is a linear operator. For $v = \frac{1}{d}M\mathbf{1}$ where $M \in \mathbb{R}^{C \times d}$, we have $\|v\|_2 \le \frac{1}{\sqrt{d}}\|M\|_F$ by Cauchy-Schwarz. Thus
\begin{equation}
\|h'(\mathbf{x}) - h(\mathbf{x})\|_2 \le \frac{1}{\sqrt{d}}\,\|H'(\mathbf{x}) - H(\mathbf{x})\|_F \le \frac{\delta(\mathbf{x})}{\sqrt{d}}\,\|U(\mathbf{x})\|_F.
\label{eq:h-drift}
\end{equation}
Decomposing the logit difference:
\begin{align*}
\hat{y}'(\mathbf{x}) - \hat{y}(\mathbf{x})
&= W' h'(\mathbf{x}) - W h(\mathbf{x}) \\
&= W' \big(h'(\mathbf{x}) - h(\mathbf{x})\big) + \big(W' - W\big)\,h(\mathbf{x}). 
\end{align*}
The result follows from the triangle inequality and submultiplicativity.
\end{proof}

The decomposition in~\eqref{eq:logit_decomposition} reveals that logit drift has two independent sources: gate changes and classifier changes. This separation is crucial. It shows that even when the classifier must adapt significantly for a new task, the gate-induced drift on old tasks remains controlled by $\delta(\mathbf{x})$.\looseness=-1

%------------------------------------------------------------------------------

Having bounded the logit drift, we now relate this perturbation to decision stability.

\begin{theorem}[Sufficient Condition for Prediction Stability]
\label{thm:stability}
Let $\hat{y}(\mathbf{x})$ and $\hat{y}'(\mathbf{x})$ denote the logits of an earlier-task sample $\mathbf{x}$ before and after learning a new task. Let the classification margin be
\[
m(\mathbf{x}) = \hat{y}_{y}(\mathbf{x}) - \max_{k\neq y} \hat{y}_{k}(\mathbf{x}),
\]
where $y$ is the true label. If
\[
\|\hat{y}'(\mathbf{x}) - \hat{y}(\mathbf{x})\|_\infty < \frac{m(\mathbf{x})}{2},
\]
then the predicted label is preserved: $\arg\max_k \hat{y}'_k(\mathbf{x}) = y$.
\end{theorem}

\begin{proof}
Let $k^* = \arg\max_{k \neq y} \hat{y}_k(\mathbf{x})$ be the runner-up class. By the margin definition and the perturbation bound:
\begin{align*}
\hat{y}'_y(\mathbf{x}) &\ge \hat{y}_y(\mathbf{x}) - \|\hat{y}' - \hat{y}\|_\infty > \hat{y}_y(\mathbf{x}) - \frac{m(\mathbf{x})}{2}, \\
\hat{y}'_{k^*}(\mathbf{x}) &\le \hat{y}_{k^*}(\mathbf{x}) + \|\hat{y}' - \hat{y}\|_\infty < \hat{y}_{k^*}(\mathbf{x}) + \frac{m(\mathbf{x})}{2}.
\end{align*}
Since $\hat{y}_y(\mathbf{x}) - \hat{y}_{k^*}(\mathbf{x}) = m(\mathbf{x})$, we have $\hat{y}'_y(\mathbf{x}) > \hat{y}'_{k^*}(\mathbf{x})$, preserving the prediction.
\end{proof}

Combining Theorems~\ref{thm:logit_drift} and~\ref{thm:stability}, we obtain an explicit sufficient condition for zero forgetting:

\begin{corollary}[Margin-Based Forgetting Guarantee]
\label{cor:margin}
For a sample $\mathbf{x}$ from a previous task with margin $m(\mathbf{x}) > 0$, the prediction is preserved if
\[
\|W'\|_2 \cdot \frac{\delta(\mathbf{x})}{\sqrt{d}}\,\|U(\mathbf{x})\|_F + \|W' - W\|_2\,\|h(\mathbf{x})\|_2 < \frac{m(\mathbf{x})}{2}.
\]
\end{corollary}

This condition reveals the interplay between representation quality and forgetting: samples with larger margins (more confident predictions) and smaller feature norms are more robust to task interference.\looseness=-1

%==============================================================================

%==============================================================================
\subsection{Expressiveness of Diagonal Gating}
%==============================================================================

While Theorems~\ref{thm:feature_drift}--\ref{thm:stability} establish the boundary for decision stability, they raise a fundamental question of plasticity: is the constrained space of diagonal gating expressive enough to achieve high accuracy on new subjects?\looseness=-1

In domain-incremental HAR, subject variability typically manifests through several mechanisms that align well with channel-wise modulation:\looseness=-1

\begin{itemize}
    \item \textbf{Biomechanical differences}: Variations in limb length, body mass, and movement patterns affect the amplitude of accelerometer and gyroscope signals in a sensor-specific manner.\looseness=-1
    
    \item \textbf{Sensor placement variability}: Even with standardized protocols, slight differences in sensor attachment result in consistent per-channel scaling and offset effects.\looseness=-1
    
    \item \textbf{Device heterogeneity}: Different sensor models or calibration states introduce multiplicative gains that are channel-specific.\looseness=-1
\end{itemize}

These observations motivate the following structural assumption:\looseness=-1

\begin{assumption}[Channel-wise Domain Shift]
\label{ass:diagonal}
For input $\mathbf{x}$ from subject $t$, the backbone features satisfy
\[
U(\mathbf{x}) \approx D(s_t)\,\bar{U}(\mathbf{x}) + \varepsilon_t(\mathbf{x}),
\]
where $\bar{U}(\mathbf{x})$ represents a canonical (subject-agnostic) feature map, $s_t \in \mathbb{R}^C_{>0}$ captures subject-specific channel scaling, and $\varepsilon_t(\mathbf{x})$ is a residual with $\|\varepsilon_t(\mathbf{x})\|_F \ll \|U_t(\mathbf{x})\|_F$.
\end{assumption}

\begin{figure}[t]
    \centering
    \includegraphics[width=\columnwidth]{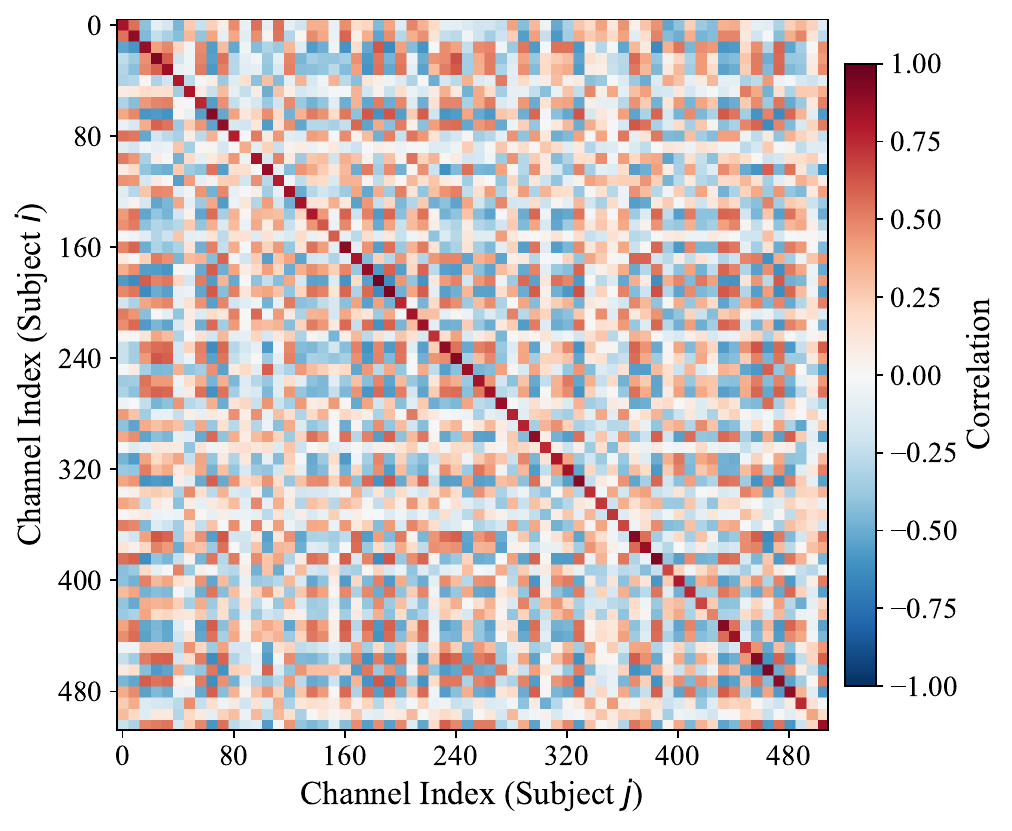}
    \caption{Cross-subject feature channel correlation matrix averaged over all subject pairs in PAMAP2. Entry $(c, c')$ shows the Pearson correlation between channel $c$ of subject $i$ and channel $c'$ of subject $j$, computed using activity-class centroids as paired observations. The strong diagonal (mean $\rho = 0.78$) indicates that channels preserve their identity across subjects, while the weak off-diagonal (mean $\rho \approx 0$) suggests minimal cross-channel mixing. This supports Assumption~\ref{ass:diagonal}: cross-subject variation is predominantly channel-wise.}\looseness=-1
    \label{fig:correlation_matrix}
\end{figure}

\begin{remark}[Empirical Support for Assumption~\ref{ass:diagonal}]
To validate this assumption, we analyzed the cross-subject correlation structure of backbone features on PAMAP2. For each subject, we computed the mean 512-dimensional feature vector (centroid) for each activity class using the frozen pretrained backbone. For each pair of subjects sharing $K$ common activities, we then computed a $C \times C$ correlation matrix where entry $(c, c')$ is the Pearson correlation between channel $c$ of subject $i$ and channel $c'$ of subject $j$ across the $K$ paired centroids. Figure~\ref{fig:correlation_matrix} shows the average correlation matrix across all 28 subject pairs.\looseness=-1

The resulting matrix exhibits a pronounced diagonal structure: the mean diagonal correlation is 0.78, indicating that the same channel responds similarly across subjects when performing the same activity. In contrast, the mean off-diagonal correlation is approximately zero, suggesting that different channels do not systematically covary across subjects. This pattern confirms that cross-subject variability in the learned feature space is predominantly channel-wise rather than involving complex cross-channel interactions.\looseness=-1
\end{remark}

Under this assumption, diagonal gating is provably sufficient for adaptation:\looseness=-1

\begin{theorem}[Expressiveness of Diagonal Gating]
\label{thm:expressiveness}
Under Assumption~\ref{ass:diagonal}, for each subject $t$ there exists a gate vector $g_t \in (0,1)^C$ and a global scalar $\alpha_t > 0$ such that the gated representation,
\[
H_t(\mathbf{x}) = D(g_t)\,U(\mathbf{x})
\]
satisfies,
\[
\|H_t(\mathbf{x}) - \alpha_t^{-1} U_t(\mathbf{x})\|_F \le \alpha_t^{-1}\|\varepsilon_t(\mathbf{x})\|_F
\]
for all $\mathbf{x}$. When $\varepsilon_t = 0$, the match is exact up to global scaling.
\end{theorem}

\begin{proof}
Choose $\alpha_t > \|s_t\|_\infty$ (e.g., $\alpha_t = (1+\epsilon)\|s_t\|_\infty$ for small $\epsilon > 0$) and define $g_t = s_t / \alpha_t$. Then $g_t \in (0,1)^C$ and
\[
D(g_t)\,U(\mathbf{x}) = \frac{1}{\alpha_t}D(s_t)\,U(\mathbf{x}).
\]
In the exact case ($\varepsilon_t = 0$), we have $U_t(\mathbf{x}) = D(s_t)U(\mathbf{x})$, so $H_t(\mathbf{x}) = \alpha_t^{-1}U_t(\mathbf{x})$ exactly.

For the approximate case:
\begin{align*}
H_t(\mathbf{x}) - \alpha_t^{-1}U_t(\mathbf{x}) 
&= \frac{1}{\alpha_t}D(s_t)U(\mathbf{x}) \\
&- \frac{1}{\alpha_t}\big(D(s_t)U(\mathbf{x}) + \varepsilon_t(\mathbf{x})\big) \\
&= -\frac{1}{\alpha_t}\varepsilon_t(\mathbf{x}).
\end{align*}
Taking norms yields the stated bound.
\end{proof}

\begin{remark}[Invariance to Global Scaling]
The factor $\alpha_t^{-1}$ is absorbed by the subsequent linear classifier without affecting decision boundaries. Specifically, if $W$ achieves correct classification on $U_t(\mathbf{x})$, then $\alpha_t W$ achieves the same on $H_t(\mathbf{x}) = \alpha_t^{-1}U_t(\mathbf{x})$. This degree of freedom allows diagonal gating to handle arbitrary positive scaling differences across subjects.\looseness=-1
\end{remark}

%------------------------------------------------------------------------------

In summary, this section examined why channel-wise gating provides a stable and effective mechanism for continual adaptation in domain-incremental HAR. By freezing the backbone and restricting updates to diagonal, bounded scaling of intermediate features, the model limits representational drift and protects knowledge from earlier subjects. At the same time, these gates are expressive enough to account for the dominant forms of cross-subject variability, enabling meaningful adaptation without introducing the destructive interference commonly observed with more flexible adapter layers. This structure naturally supports a favorable stability–plasticity trade-off, which aligns with the empirical performance observed in our experiments.\looseness=-1

A key limitation of this analysis is its reliance on domain shifts that are approximately channel-wise, an assumption that holds well in HAR but may not generalize to domains with richer cross-channel interactions. Additionally, the stability guarantees depend on non-trivial prediction margins, which can vary across activities and subjects. While the gating mechanism performs well under typical HAR conditions, more expressive forms of adaptation may be needed when feature shifts deviate significantly from this structure.\looseness=-1

%% file: Sections/05_experimntal_setup.tex
\section{Experimental Setup}

In this section, we describe the HAR benchmark datasets we used, followed by the implementation details and system configuration. We will also define the evaluation metrics used to assess our experiments, details of which will be elaborated in the next section.\looseness=-1

\subsection{Benchmarks}
\subsubsection{PAMAP2 Physical Activity Monitoring~{\normalfont\cite{pamap2_physical_activity_monitoring_231}}}
The PAMAP2 dataset contains inertial measurements from 8 subjects performing 12 predefined physical activities. Although the original release includes 6 additional optional activities, these were excluded from our evaluation due to inconsistent coverage across participants. The dataset provides accelerometer, gyroscope, and magnetometer signals, totaling 9 synchronized sensor channels sampled at 100 Hz.\looseness=-1

\subsubsection{Daily and Sports Activities~{\normalfont\cite{daily_and_sports_activities_256}}}
The Daily and Sports Activities (DSA) dataset comprises recordings from 8 subjects executing 18 activities of daily living and sport-related motions. The dataset includes accelerometer, gyroscope, and magnetometer signals collected across 15 channels, sampled at 25 Hz.\looseness=-1

\subsubsection{UCI Human Activity Recognition Using Smartphones~{\normalfont\cite{human_activity_recognition_using_smartphones_240}}}
The UCI HAR dataset consists of motion sensor recordings from 30 participants performing 6 common activities. It provides accelerometer and gyroscope measurements across 9 sensor channels, sampled at 50 Hz. The dataset includes both raw signals and preprocessed segments commonly used in human activity recognition research.\looseness=-1

\subsubsection{RealWorld HAR~{\normalfont\cite{sztyler2016realworld}}}
The RealWorld dataset contains recordings from 15 subjects performing 8 activities. The subjects span a broad demographic range in age, height, and weight, and the data were collected in everyday, in-the-wild settings rather than a fixed laboratory protocol. Each subject wore sensors at seven body locations simultaneously, sampled at 50 Hz.\looseness=-1

Relative to the other three benchmarks, which are collected under more controlled conditions, RealWorld is recorded during free, unscripted execution of each activity and exhibits greater variability in subject demographics and in how sensors are worn and oriented across the body. This makes it a useful complement to the laboratory-style datasets, providing a less controlled setting in which the marginal distribution of the sensor signals departs more substantially across subjects.\looseness=-1

\subsection{Implementation Details}
\subsubsection{Model Architecture}

Our CNN architecture consists of an initial 1×1 convolutional embedding layer that projects input channels to 256 dimensions, followed by four residual stacks with progressive channel expansion and temporal downsampling. Table~\ref{tab:architecture} details the layer configurations.\looseness=-1

\begin{table}[h]
\centering
\caption{CNN architecture with 4 residual stacks. Each stack contains multiple convolutional layers with batch normalization and ReLU activations.\looseness=-1}
\label{tab:architecture}
\small
\begin{tabular}{lccc}
\toprule
\textbf{Stack} & \textbf{Channels} & \textbf{Layers} & \textbf{Stride} \\
\midrule
Embedding & $C_0 \to 256$ & Conv1d (k=1) & 1 \\
Stack 1 & 256 & 4 conv layers (k=1,3,3,1) & 1 \\
Stack 2 & $256 \to 384$ & 4 conv layers (k=1,5,3,1) & 1 \\
Stack 3 & $384 \to 512$ & 4 conv layers (k=1,3,3,1) & 2 \\
Stack 4 & 512 & 4 conv layers (k=1,3,3,1) & 2 \\
\midrule
Pooling & 512 & AdaptiveAvgPool1d & -- \\
Classifier & $512 \to K$ & Linear & -- \\
\bottomrule
\end{tabular}
\end{table}

Each stack follows a bottleneck design: 1×1 compression, large-kernel feature extraction (3×3 or 5×5), 1×1 expansion, with residual connections around each stack and batch normalization after every convolutional layer. Global average pooling aggregates the final 512-channel feature map into a single vector, followed by a linear classifier. The total model contains approximately 7.4 million parameters.\looseness=-1

\subsubsection{Pre-training Procedure}

We pre-trained the CNN backbone on the WISDM dataset~\cite{wisdm}, a smartphone-based HAR dataset with 18 activity classes collected from 51 subjects. Raw tri-axial accelerometer data sampled at 20 Hz was segmented into fixed-length windows of 200 timesteps (10 seconds) with 50\% overlap (step size of 100 timesteps). Per-subject z-score normalization was applied independently to each axis to account for inter-subject variability in sensor calibration and placement.\looseness=-1

The model was trained for 300 epochs using the Adam optimizer with an initial learning rate of 0.001, cosine annealing schedule decaying to $10^{-6}$, mild weight decay of 0.0001, gradient clipping with max norm 1.0, and batch size 64. Cross-entropy loss was used, and the best model checkpoint based on validation accuracy was saved. Early stopping with patience of 20 epochs prevented overfitting.\looseness=-1

Once the pre-training is done, after each stack's residual block and batch normalization, we optionally insert a channel gate module. Each gate computes a channel-wise scaling vector $g \in (0,1)^C$ via global average pooling, a two-layer bottleneck MLP with reduction ratio $r=8$ (hidden dimension $\max(C/8, 16)$), ReLU activation, and sigmoid output. The gate is applied via element-wise multiplication. For a 512-channel stack, each gate contains $(512 \times 64) + (64 \times 512) = 65{,}536$ parameters. \looseness=-1

\subsubsection{Data Preprocessing} 

For each target dataset, raw sensor data was segmented using a sliding window with 50\% overlap. Window sizes matched the temporal resolution of each dataset: UCI-HAR used 128 timesteps (2.56s at 50 Hz), PAMAP2 used 200 timesteps (2s at 100 Hz), DSA used 125 timesteps (5s at 25 Hz), and RealWorld used 200 timesteps (4s at 50 Hz). Per-subject z-score normalization was applied independently to each channel before windowing. To match the pretrained model's expected input length (200 timesteps), windows were resized via linear interpolation.\looseness=-1

Each subject defines one task in the continual learning sequence. Within each subject's data, an 80/20 stratified train/test split by activity class prevents data leakage. Subjects are presented in random order, with 10 different random permutations tested per experiment to account for task-order variability.\looseness=-1

\subsubsection{Training Procedure}

For each subject (task), we train only the gate parameters (if present) and classifier while keeping the backbone frozen. Training proceeds for up to 100 epochs with early stopping. We use the Adam optimizer with learning rate 0.001, L2 regularization ($\lambda=0.0001$), batch size 64, and gradient clipping (max norm 1.0). After training on subject $t$, the model is evaluated on test sets from all subjects $1, \dots, t$ to measure forgetting. No data replay or task-specific regularization is used, and stability emerges purely from the frozen backbone and bounded gate updates.\looseness=-1

\subsection{Evaluation Metrics}
\subsubsection{Final Accuracy (FA)}
This metric evaluates the model’s overall performance across all subjects after sequential training has been completed. It directly reflects the model’s ability to retain what it has learned throughout the entire training process.\looseness=-1

\begin{equation}
\text{FA} = \frac{1}{T}\sum_{t=1}^T A_{t, T}
\end{equation}

where $A_{t, T}$ is the model’s test accuracy on subject $t$ after it has been trained on all $T$ subjects.\looseness=-1

\subsubsection{Forgetting Measure (FM)}
This metric quantifies the degree of forgetting that occurs over training. After completing training on all subjects, FM captures how much the model’s performance on earlier subjects has degraded relative to its best performance during training.\looseness=-1

\begin{equation}
\text{FM} = \dfrac{1}{T}\sum_{t=1}^T \left( \max_{t' \in \{1, \ldots, T\}} A_{t, t'} - A_{t, T} \right)
\end{equation}

where $A_{t, t'}$ denotes the test accuracy on subject $t$ after training up through subject $t'$.\looseness=-1

\subsubsection{Learning Accuracy (LA)}
This metric evaluates the model’s plasticity, measuring how effectively it acquires new information when each subject is introduced. LA reflects the model’s immediate performance on a newly encountered subject before any interference from subsequent tasks.\looseness=-1

\begin{equation}
\text{LA} = \frac{1}{T}\sum_{t=1}^{T} A_{t, t}
\end{equation}

where $A_{t,t}$ denotes the test accuracy on subject $t$ immediately after the model has been trained on that same subject.\looseness=-1

%% file: Sections/06_experimental_results.tex
\section{Experimental Results}

In this section, we present the different types of experiments we conducted and their results. Our experiments include comparing our method against the baseline algorithms, as well as ablation studies that quantify the contribution of each component of the proposed method.\looseness=-1

\subsection{Comparison with Replay-Free Baselines}

We first evaluate our approach against continual learning methods that, like ours, do not store past samples. This comparison isolates the effectiveness of different adaptation mechanisms under identical memory constraints. All baseline methods use a trainable backbone initialized from the same WISDM-pretrained weights, following their original formulations.\looseness=-1

\textbf{Regularization-based methods.} Elastic Weight Consolidation (EWC)~\cite{kirkpatrick2017ewc} penalizes changes to parameters deemed important for previous tasks, using the Fisher information matrix to estimate importance. Learning without Forgetting (LwF)~\cite{li2017lwf} preserves prior knowledge through knowledge distillation, using the previous model's outputs as soft targets during training on new tasks.\looseness=-1

\textbf{Architecture-based methods.} Hard Attention to Task (HAT)~\cite{serrà2018overcomingcatastrophicforgettinghard} learns binary attention masks over network units, with each task receiving a dedicated mask that protects important units from modification during subsequent learning.\looseness=-1

Table~\ref{tab:replay_free} summarizes the results. Our gating approach consistently achieves the highest final accuracy and lowest forgetting across all four benchmarks among replay-free methods. On PAMAP2, we outperform the next-best method (HAT) by 9.9 percentage points in final accuracy while reducing forgetting by 12.3 points. Notably, our approach also exhibits substantially lower variance across task orderings (std of 2.5\% vs 6--10\% for baselines), indicating more robust performance regardless of the order in which subjects are encountered. These results confirm that structured diagonal adaptation provides a more effective stability-plasticity tradeoff than either regularization-based constraints or binary masking.\looseness=-1

\input{Tables/2_replay_free_tbl}

\subsection{Comparison with Replay-Based Methods}

Replay-based methods represent a fundamentally different tradeoff: they achieve stronger performance by maintaining a memory buffer of past samples, incurring additional storage costs and, in the case of sensor data, potential privacy implications. We include this comparison to contextualize our results within the broader continual learning landscape.\looseness=-1

Table~\ref{tab:replay} presents results for Dark Experience Replay (DER)~\cite{buzzega2020darkexperiencegeneralcontinual} and its extension DER++, both using a buffer of 500 samples. As expected, replay-based methods achieve higher final accuracy than replay-free approaches, with DER++ reaching 90.1\% on PAMAP2 compared to 77.7\% for our method. However, this performance gap must be weighed against several practical considerations, which we now quantify directly rather than describe qualitatively.\looseness=-1

First, replay methods require training the full backbone, updating 100\% of model parameters per task compared to only 2.5\% for our approach. As reported in Table~\ref{tab:efficiency}, this difference is concrete: our method updates roughly 0.19M parameters per task versus 7.44M for full-backbone methods. Because no gradient flows through the frozen backbone, the measured per-step training cost falls to 0.60× that of full-backbone training and peak training memory drops by about a quarter (e.g., 561 MB vs. 764 MB on PAMAP2), since optimizer state and activation gradients are not maintained for the backbone. All methods share an identical inference forward pass (1.67 GFLOPs per window), so these savings are in training cost rather than inference latency.\looseness=-1

Second, storing 500 sensor samples requires persistent on-device memory that scales with buffer size. Concretely, each buffered window is the raw multi-sensor segment together with its label and stored logits; at a 500-sample buffer this amounts to 10.3 MB on PAMAP2, 17.2 MB on DSA, and 2.3 MB on UCI-HAR (Table~\ref{tab:efficiency}). For resource-constrained IoT devices such as smartwatches or fitness bands, this memory overhead may be prohibitive, particularly when multiple applications compete for limited storage.\looseness=-1

Third, and perhaps most critically for IoT deployments, retaining raw accelerometer and gyroscope data raises privacy concerns. Movement patterns can reveal sensitive information about health conditions, daily routines, and behavioral profiles~\cite{lane2010survey}. Because the buffer stores raw sensor windows rather than compressed embeddings, this storage cost is simultaneously a privacy exposure: every retained window is fully reconstructable motion data. For applications such as elderly care monitoring or remote patient tracking, such data retention may conflict with privacy regulations or user expectations. Our replay-free approach achieves strong performance without storing any historical data, making it suitable for privacy-sensitive edge deployments.\looseness=-1

Taken together, these measurements make explicit when the accuracy--retention trade-off favors each approach, and also delimit the deployment assumptions of our framework. Replay buys higher final accuracy (e.g., +12 points on PAMAP2) at the cost of persistently retained raw sensor data; it is the better choice where storage and raw-data retention are unconstrained and peak accuracy is paramount. Our replay-free method is preferable in the opposite regime (tight storage budgets, constrained training energy, or restrictions on raw-sensor retention from privacy regulation) which are precisely the defining conditions of on-device wearable deployment. We emphasize, however, that this trade-off is contingent on those conditions rather than universal: in centralized or privacy-permissive settings where a buffer is acceptable, replay (or the hybrid of Section~\ref{subsec:hybrid}) remains the stronger option. The framework also assumes that subject-induced shift is approximately channel-wise and that a sufficiently diverse pretrained backbone is available; where these assumptions weaken (for example under large cross-channel distribution shift, or when the source and target domains differ substantially) the advantage of frozen, diagonally gated adaptation may narrow, and more expressive adaptation could be required.\looseness=-1

We note that gating and replay address forgetting through complementary mechanisms: one architectural, one data-driven. In the following subsections, we will explore a hybrid approach that combines our gating mechanism with experience replay, demonstrating that these techniques can be effectively integrated when deployment constraints permit data retention.\looseness=-1

\input{Tables/3_replay_based_tbl}

% =====================================================================
% NEW SUBSECTION (R2.3 efficiency/privacy evidence + R2.7 on-device measurements)
% =====================================================================
\subsection{Efficiency, Storage, and Deployment Cost}
\label{subsec:efficiency}

To support the efficiency and privacy claims above with concrete numbers rather than qualitative argument, we measured the per-task cost of each method family on identical hardware. Trainable parameters, retained storage, and inference FLOPs are exact and hardware-independent; training time and peak memory are inherently device-specific, so we report training time as a ratio to full-backbone training measured on the same device, which transfers across platforms more reliably than absolute timings. Table~\ref{tab:efficiency} summarizes the results. Our gated method updates only about 2.5\% of parameters, retains no data, trains at roughly 0.60× the per-step cost of full-backbone methods, and reduces peak training memory by about a quarter, while sharing the identical inference forward pass of all methods. The retained-storage column also quantifies the privacy exposure of replay discussed above: between 2.3 and 17.2 MB of raw, fully reconstructable sensor windows per device at a 500-sample buffer, versus zero for our approach.\looseness=-1

\input{Tables/4_deployment_tbl}

\subsection{Effect of Backbone Freezing}

A central design choice in our approach is the use of a frozen pretrained backbone. To assess the impact of this decision, we compare models with trainable versus frozen backbones, both with and without gating mechanisms. Table~\ref{tab:freezing_gates} presents results across all four benchmarks.\looseness=-1

\paragraph{Freezing without gates.} Comparing \textit{Base} (trainable backbone, no gates) to \textit{Pretrained} (frozen backbone, no gates), we observe that freezing the backbone consistently reduces forgetting, demonstrating that a stable feature representation limits catastrophic interference. On PAMAP2, FM drops dramatically from 39.7\% to 17.5\%, a 56\% relative reduction. However, this stability comes at the cost of reduced learning accuracy: LA drops from 96.5\% to 93.9\%, as the frozen backbone cannot fully adjust to subject-specific variations. Despite this, the net effect on final accuracy is strongly positive: FA improves from 56.7\% to 76.5\% due to the substantial reduction in forgetting. We note that this favorable balance presumes the pretrained representation transfers reasonably to the target subjects; under a pronounced source--target mismatch the frozen backbone would forgo more plasticity than gating can recover, and partial unfreezing or a more expressive adapter could become preferable.\looseness=-1

\paragraph{Freezing with gates.} When gating mechanisms are introduced, the benefits compound. Comparing \textit{Base + Gates} (trainable backbone with gates) to \textit{Pretrained + Gates} (frozen backbone with gates), we find that freezing still reduces forgetting (FM drops from 31.3\% to 16.2\% on PAMAP2), and final accuracy improves substantially (FA increases from 64.9\% to 77.7\%). Gates provide the frozen backbone with targeted adaptation capacity, allowing it to match or exceed the plasticity of trainable backbones while maintaining the stability advantages of frozen representations.\looseness=-1

\subsection{Effect of Gating Mechanisms}

We next isolate the contribution of channel-wise gates by comparing models with and without gating, holding the backbone training strategy fixed. Results are shown in Table~\ref{tab:freezing_gates}.\looseness=-1

\paragraph{Gates on trainable backbones.} Comparing \textit{Base} (no gates) to \textit{Base + Gates}, we observe meaningful improvement from gating even when the backbone is trainable. On PAMAP2, FA improves from 56.7\% to 64.9\%, and FM decreases from 39.7\% to 31.3\%. This suggests that gates provide useful regularization even when the entire network can adapt, though the effect is more modest than with frozen backbones.\looseness=-1

\paragraph{Gates on frozen backbones.} The effect of gating becomes most pronounced when the backbone is frozen. Comparing \textit{Pretrained} (no gates) to \textit{Pretrained + Gates}, we see improvements in final accuracy on PAMAP2 (FA increases from 76.5\% to 77.7\%) while maintaining low forgetting. This demonstrates that gates provide the critical adaptation capacity needed when the backbone is frozen, enabling subject-specific feature modulation without disrupting learned representations.\looseness=-1

\paragraph{Gates reduce variance.} Beyond mean performance, we observe that gating reduces variability across task orderings, particularly on PAMAP2, where the standard deviation of FA decreases from 8.9\% (Base) to 2.5\% (Pretrained + Gates). This suggests that gates help stabilize learning across different subject sequences.\looseness=-1

\input{Tables/5_backbone_freezing_tbl}

\subsection{Gating vs. Stacked Trainable Layers}

A central theoretical claim of our work is that gated adaptation, which performs feature \emph{selection} through channel-wise magnitude scaling, offers superior stability compared to stacked trainable layers, which perform feature \emph{generation} through arbitrary linear transformations. To test this hypothesis, we compare our gating approach against architectures with varying numbers of fully connected layers inserted between the frozen backbone and classifier. Each stacked layer consists of a linear transformation (512 → 512), ReLU activation, and dropout (0.3).\looseness=-1

Table~\ref{tab:gates_vs_stacked} presents results across all four benchmarks. On PAMAP2, we observe a clear pattern as the number of stacked layers increases: learning accuracy (LA) improves monotonically from 93.9\% (no layer) to 95.3\% (3 layers), indicating that additional capacity helps the model fit each new subject. However, forgetting (FM) also increases substantially from 17.5\% to 25.3\%, demonstrating that this added capacity comes at a steep cost to stability. The net effect on final accuracy is negative: FA drops from 76.5\% to 70.1\% as forgetting outweighs the plasticity gains.\looseness=-1

In contrast, our gating approach achieves FM of only 16.2\% while maintaining competitive FA of 77.7\%, striking a balance between stability and plasticity. This pattern holds across datasets. On DSA, three stacked layers reach FM of 25.6\% compared to 19.6\% for gates, while achieving lower FA (73.0\% vs 78.7\%). On UCI-HAR, the gap is similarly pronounced.\looseness=-1

\paragraph{Capacity is not the answer.} An important observation from these results is that simply adding more trainable parameters does not reduce forgetting; in fact, it exacerbates it. The stacked 3-layer architecture contains significantly more parameters than our gating mechanism, yet performs substantially worse on forgetting. This confirms that the \emph{structure} of adaptation matters more than the sheer number of parameters, and that bounded, interpretable transformations (diagonal scaling) are preferable to unconstrained transformations (dense layers) in continual learning settings.\looseness=-1

\input{Tables/6_stacked_layer_tbl}

\input{Tables/7_KD_tbl}

\subsection{Effect of the Number of Gates}
\label{subsec:gate_sweep}

Our default inserts a channel gate after every backbone stack. To verify that the stability--plasticity benefit comes from the bounded, diagonal form of the gating rather than from a particular amount of gating capacity, we vary the number of gates inserted into the frozen backbone, adding them to the first $N$ stacks for $N \in \{0,1,2,3,4\}$. This mirrors the stacked-layer ablation of Table~\ref{tab:gates_vs_stacked}, but along the feature-\textit{selection} axis instead of feature-\textit{generation}; the $N{=}0$ and $N{=}4$ settings coincide with the \textit{Pretrained} and \textit{Pretrained + Gates (Ours)} configurations.

Table~\ref{tab:gate_sweep} reports the results across all four benchmarks. In contrast to stacking trainable layers, where adding capacity steadily increased forgetting and reduced final accuracy, varying the number of gates produces no such degradation: final accuracy stays within a narrow band across all gate counts, and the full configuration ($N{=}4$) attains the best or near-best final accuracy on every dataset while keeping forgetting low. Intermediate gate counts are competitive but do not consistently improve over the full configuration, indicating that the advantage of the mechanism is not sensitive to a finely tuned number of gates. These results support our main claim that bounded channel-wise selection, rather than added adaptation capacity per se, is what yields a favorable stability--plasticity trade-off.\looseness=-1

\input{Tables/8_gate_sweep_tbl}
\subsection{Improvement Strategies}

Having established that task-free gates with frozen backbones provide superior stability-plasticity tradeoffs, we now examine three extensions: (1) knowledge distillation to further reduce forgetting, (2) task-aware gates that leverage task identity at test time, and (3) combining our approach with experience replay.\looseness=-1

\subsubsection{Knowledge Distillation}

Knowledge distillation (KD)\cite{hinton2015distillingknowledgeneuralnetwork} is a widely used technique in continual learning that preserves the model's predictions on previous tasks by matching the output distributions of the updated model to those of the model before the update. We augment our gating approach with KD loss:\looseness=-1
\[
\mathcal{L}_{\text{total}} = \alpha \mathcal{L}_{\text{CE}}(y, \hat{y}) + (1-\alpha) \mathcal{L}_{\text{KD}}(\hat{y}, \hat{y}_{\text{old}}),
\]
where $\mathcal{L}_{\text{CE}}$ is cross-entropy on the current task, $\mathcal{L}_{\text{KD}}$ is the KL divergence between current and previous model outputs, and $\alpha=0.5$ balances the two objectives.\looseness=-1

Table~\ref{tab:extensions} shows that KD consistently reduces forgetting across all datasets. On PAMAP2, KD brings FM from 16.2\% down to 12.7\%, a 22\% relative reduction. However, this comes at the cost of reduced plasticity: LA drops from 93.8\% to 90.2\%, indicating that the model becomes more conservative in adapting to new subjects. The net effect on FA is minimal (77.7\% vs 77.5\%), suggesting that the forgetting reduction and plasticity loss roughly cancel out. While KD is effective at preserving stability, it may be too restrictive for scenarios where rapid adaptation to new subjects is critical.\looseness=-1

\subsubsection{Task-Aware Gates}

In our task-free setting, a single set of shared gates is trained across all subjects, and the model must infer subject identity implicitly from input statistics. An alternative design is to maintain separate gates for each subject and provide explicit task identity at test time. This \emph{task-aware} approach represents an oracle setting where perfect task segmentation is available.\looseness=-1

As expected, task-aware gates achieve lower forgetting than task-free gates. On PAMAP2, FM drops from 16.2\% (task-free) to 12.3\% (task-aware), and FA improves from 77.7\% to 78.9\%. The improvement is consistent but modest (2-4 percentage points), suggesting that task-free gates already capture much of the benefit of subject-specific adaptation without requiring explicit task boundaries. This is encouraging for real-world deployments, where task identity may not be available or may be ambiguous.\looseness=-1

Task-aware gates require storing one set of gate parameters per subject, increasing memory linearly with the number of tasks. For UCI-HAR with 30 subjects, this means 30× the gate parameters. In contrast, task-free gates require only a single set regardless of the number of subjects. For resource-constrained devices such as smartphones or wearables, task-free gates are thus the more practical choice.\looseness=-1

\subsubsection{Combining Gates with Experience Replay}
\label{subsec:hybrid}

To demonstrate that our gating mechanism provides benefits orthogonal to replay-based methods, we evaluate a variant that combines gates with Dark Experience Replay (DER). This hybrid approach maintains a small memory buffer (500 samples) and applies the DER++ loss while training only the gates and classifier, keeping the backbone frozen.\looseness=-1

Results show that Gates + DER achieves substantial improvements over gates alone. On PAMAP2, FA improves from 77.7\% to 84.3\%, while FM drops dramatically from 16.2\% to 6.1\%, a 62\% relative reduction in forgetting. Similar patterns emerge on DSA (FA: 78.7\% → 85.7\%, FM: 19.6\% → 9.6\%) and UCI-HAR (FA: 80.3\% → 85.0\%, FM: 14.5\% → 6.1\%). Importantly, this performance is competitive with standard DER++ while updating only 2.5\% of model parameters (gates + classifier vs. full backbone). This demonstrates that architectural gating and replay provide complementary mechanisms for preventing forgetting, and can be effectively combined when memory constraints permit storing past samples.\looseness=-1

However, maintaining a replay buffer may be impractical for on-device deployment scenarios. Storing raw sensor data on smartphones or wearables raises storage concerns. More critically, retaining historical movement patterns poses privacy risks, as such data can reveal sensitive information about users' daily routines and health conditions. For these reasons, our replay-free gating approach remains the preferred choice for privacy-sensitive edge deployments, achieving strong performance without requiring any data retention.\looseness=-1

%% file: Tables/2_replay_free_tbl.tex
\begin{table*}[t]
\centering
\caption{Comparison with replay-free continual learning methods (mean \mbox{$\pm$} std over 10 permutations). FA: Final Accuracy (\%), FM: Forgetting Measure (\%), LA: Learning Accuracy (\%)}
\label{tab:replay_free}
\footnotesize
\setlength{\tabcolsep}{2.5pt}
\resizebox{\textwidth}{!}{%
\begin{tabular}{l|ccc|ccc|ccc|ccc}
\toprule
& \multicolumn{3}{c|}{\textbf{PAMAP2}} 
& \multicolumn{3}{c|}{\textbf{DSA}} 
& \multicolumn{3}{c|}{\textbf{UCI-HAR}} 
& \multicolumn{3}{c}{\textbf{RealWorld}} \\
& \multicolumn{3}{c|}{8 subj., 12 act.} 
& \multicolumn{3}{c|}{8 subj., 18 act.} 
& \multicolumn{3}{c|}{30 subj., 6 act.} 
& \multicolumn{3}{c}{15 subj., 8 act.} \\
\cmidrule(lr){2-4} \cmidrule(lr){5-7} \cmidrule(lr){8-10} \cmidrule(lr){11-13}
& FA $\uparrow$ & FM $\downarrow$ & LA $\uparrow$ & FA $\uparrow$ & FM $\downarrow$ & LA $\uparrow$ & FA $\uparrow$ & FM $\downarrow$ & LA $\uparrow$ & FA $\uparrow$ & FM $\downarrow$ & LA $\uparrow$ \\
\midrule
EWC         & $66.2{\pm}7.5$ & $29.8{\pm}7.5$ & $95.6{\pm}0.6$ & $68.5{\pm}5.2$ & $30.1{\pm}5.4$ & $98.5{\pm}0.3$ & $76.2{\pm}8.5$ & $22.3{\pm}8.8$ & $98.9{\pm}0.5$ & $56.4{\pm}8.1$ & $38.6{\pm}8.0$ & $89.9{\pm}0.7$ \\
LwF         & $64.8{\pm}10.0$ & $31.3{\pm}10.3$ & $96.8{\pm}0.5$ & $69.8{\pm}6.0$ & $29.2{\pm}5.6$ & $99.0{\pm}0.3$ & $74.7{\pm}11.3$ & $23.6{\pm}11.7$ & $98.4{\pm}0.7$ & $55.2{\pm}9.6$ & $39.5{\pm}9.8$ & $90.5{\pm}0.6$ \\
HAT         & $67.8{\pm}8.2$ & $28.5{\pm}8.4$ & $96.6{\pm}0.4$ & $70.2{\pm}5.8$ & $28.3{\pm}5.6$ & $99.1{\pm}0.2$ & $77.5{\pm}6.2$ & $21.0{\pm}6.4$ & $98.7{\pm}0.4$ & $57.9{\pm}7.0$ & $36.8{\pm}7.2$ & $90.6{\pm}0.5$ \\
\midrule
\textbf{Ours} & ${77.7{\pm}2.5}$ & ${16.2{\pm}2.6}$ & $93.8{\pm}0.4$ &
${78.7{\pm}1.7}$ & ${19.6{\pm}1.7}$ & $98.3{\pm}0.3$ &
${80.3{\pm}5.3}$ & ${14.5{\pm}5.4}$ & $94.9{\pm}1.2$ &
${65.6{\pm}3.2}$ & ${26.4{\pm}3.3}$ & $86.5{\pm}0.8$ \\
\midrule
Upper Bound & $92.8{\pm}0.9$ & $0.0{\pm}0.0$ & $92.8{\pm}0.9$ & $97.4{\pm}0.6$ & $0.0{\pm}0.0$ & $97.4{\pm}0.6$ & $94.1{\pm}0.6$ & $0.0{\pm}0.0$ & $94.1{\pm}0.6$ & $79.9{\pm}0.9$ & $0.0{\pm}0.0$ & $79.9{\pm}0.9$ \\
\bottomrule
\end{tabular}%
}
\end{table*}

%% file: Tables/3_replay_based_tbl.tex
\begin{table*}[t]
\centering
\caption{Comparison with replay-based methods (mean \mbox{$\pm$} std over 10 permutations). Params: percentage of model parameters updated per task. Buffer: number of stored samples.}\looseness=-1
\label{tab:replay}
\footnotesize
\setlength{\tabcolsep}{2.5pt}
\resizebox{\textwidth}{!}{%
\begin{tabular}{l|ccc|ccc|ccc|ccc}
\toprule
& \multicolumn{3}{c|}{\textbf{PAMAP2}} 
& \multicolumn{3}{c|}{\textbf{DSA}} 
& \multicolumn{3}{c|}{\textbf{UCI-HAR}} 
& \multicolumn{3}{c}{\textbf{RealWorld}} \\
\cmidrule(lr){2-4} \cmidrule(lr){5-7} \cmidrule(lr){8-10} \cmidrule(lr){11-13}
& FA $\uparrow$ & FM $\downarrow$ & LA $\uparrow$ & FA $\uparrow$ & FM $\downarrow$ & LA $\uparrow$ & FA $\uparrow$ & FM $\downarrow$ & LA $\uparrow$ & FA $\uparrow$ & FM $\downarrow$ & LA $\uparrow$ \\
\midrule
DER    & $88.5{\pm}0.6$ & $8.1{\pm}0.8$ & $95.6{\pm}0.4$ & $92.3{\pm}1.8$ & $6.9{\pm}1.8$ & $98.2{\pm}0.2$ & $89.5{\pm}1.9$ & $9.8{\pm}2.0$ & $98.4{\pm}0.2$ & $75.6{\pm}1.4$ & $16.3{\pm}1.5$ & $90.4{\pm}0.4$ \\
DER++  & ${90.1{\pm}1.6}$ & ${6.3{\pm}1.4}$ & $95.4{\pm}0.4$ & ${94.1{\pm}1.7}$ & ${5.0{\pm}1.6}$ & $98.0{\pm}0.2$ & ${90.5{\pm}2.6}$ & ${8.8{\pm}2.6}$ & $99.4{\pm}0.2$ & ${77.1{\pm}1.9}$ & ${14.4{\pm}1.8}$ & $90.7{\pm}0.3$ \\
\midrule
Ours   & $77.7{\pm}2.5$ & $16.2{\pm}2.6$ & $93.8{\pm}0.4$ & $78.7{\pm}1.7$ & $19.6{\pm}1.7$ & $98.3{\pm}0.3$ & $80.3{\pm}5.3$ & $14.5{\pm}5.4$ & $94.9{\pm}1.2$ & $65.6{\pm}3.2$ & $26.4{\pm}3.3$ & $86.5{\pm}0.8$ \\
\bottomrule
\end{tabular}%
}
\end{table*}

%% file: Tables/4_deployment_tbl.tex
\begin{table*}[t]
\centering
\caption{Efficiency, storage, and deployment cost. Trainable parameters and retained storage are exact; inference FLOPs are identical across methods (same forward pass); training-step time is reported relative to full-backbone training, and peak training memory is measured on the same GPU. Retained storage is given per dataset at a 500-sample buffer.}
\label{tab:efficiency}
\footnotesize
\setlength{\tabcolsep}{4pt}
\begin{tabular}{l|cc|cc|cccc}
\toprule
& Trainable & Train step & Inf.\ FLOPs & Peak train & \multicolumn{4}{c}{Retained storage (MB)} \\
Method & params & (rel.) & /window & mem.\ (rel.) & PAMAP2 & DSA & UCI & RealWorld \\
\midrule
Ours (gates, frozen)     & 0.19 M (2.5\%) & 0.60× & 1.67 G & $\sim$0.74× & 0 & 0 & 0 & 0 \\
Full-backbone (EWC, HAT) & 7.44 M (100\%) & 1.0×  & 1.67 G & 1.0×        & 0 & 0 & 0 & 0 \\
Replay (DER, DER++)      & 7.44 M (100\%) & 1.0×  & 1.67 G & 1.0×        & 10.3 & 17.2 & 2.3 & 10.3 \\
Gates + DER              & 0.19 M (2.5\%) & 0.60× & 1.67 G & $\sim$0.74× & 10.3 & 17.2 & 2.3 & 10.3 \\
\bottomrule
\end{tabular}

\vspace{2pt}
{\footnotesize Per-sample buffer cost: 21.2 / 35.2 / 4.7 / 21.1 KB for PAMAP2 / DSA / UCI-HAR / RealWorld. EWC and HAT retain no samples but store auxiliary state (Fisher matrix / task masks).}
\end{table*}

%% file: Tables/5_backbone_freezing_tbl.tex
\begin{table*}[t]
\centering
\caption{Effect of backbone freezing and gating mechanisms across HAR benchmarks (mean \mbox{$\pm$} std over 10 permutations).}\looseness=-1
\label{tab:freezing_gates}
\footnotesize
\setlength{\tabcolsep}{2.5pt}
\resizebox{\textwidth}{!}{%
\begin{tabular}{l|ccc|ccc|ccc|ccc}
\toprule
& \multicolumn{3}{c|}{\textbf{PAMAP2}} 
& \multicolumn{3}{c|}{\textbf{DSA}} 
& \multicolumn{3}{c|}{\textbf{UCI-HAR}} 
& \multicolumn{3}{c}{\textbf{RealWorld}} \\
& \multicolumn{3}{c|}{8 subj., 12 act.} 
& \multicolumn{3}{c|}{8 subj., 18 act.} 
& \multicolumn{3}{c|}{30 subj., 6 act.} 
& \multicolumn{3}{c}{15 subj., 8 act.} \\
\cmidrule(lr){2-4} \cmidrule(lr){5-7} \cmidrule(lr){8-10} \cmidrule(lr){11-13}
& FA $\uparrow$ & FM $\downarrow$ & LA $\uparrow$ & FA $\uparrow$ & FM $\downarrow$ & LA $\uparrow$ & FA $\uparrow$ & FM $\downarrow$ & LA $\uparrow$ & FA $\uparrow$ & FM $\downarrow$ & LA $\uparrow$ \\\midrule
\multicolumn{13}{l}{\textit{Trainable Backbone}} \\
Base & $56.7{\pm}8.9$ & $39.7{\pm}8.9$ & $96.5{\pm}0.3$ & $68.7{\pm}3.9$ & $30.4{\pm}4.1$ & $99.1{\pm}0.3$ & $75.6{\pm}8.6$ & $23.3{\pm}8.8$ & $98.9{\pm}0.6$ & $51.4{\pm}7.8$ & $40.9{\pm}7.9$ & $90.6{\pm}0.4$ \\
Base + Gates & $64.9{\pm}10.0$ & $31.3{\pm}10.3$ & $96.2{\pm}0.5$ & $69.8{\pm}6.0$ & $29.2{\pm}5.9$ & $99.0{\pm}0.3$ & $74.7{\pm}11.3$ & $23.7{\pm}11.7$ & $98.4{\pm}0.7$ & $56.3{\pm}9.1$ & $36.6{\pm}9.3$ & $90.3{\pm}0.5$ \\
\midrule
\multicolumn{13}{l}{\textit{Frozen Backbone}} \\
Pretrained &
$76.5{\pm}4.0$ & $17.5{\pm}4.0$ & $93.9{\pm}0.4$ &
$75.7{\pm}2.7$ & $22.1{\pm}2.7$ & $97.9{\pm}0.2$ &
$80.1{\pm}9.4$ & $13.9{\pm}9.5$ & $94.0{\pm}1.5$ &
$64.4{\pm}3.6$ & $27.3{\pm}3.6$ & $86.2{\pm}0.9$ \\
\textbf{Pretrained + Gates (Ours)} &
${77.7{\pm}2.5}$ & ${16.2{\pm}2.6}$ & $93.8{\pm}0.4$ &
${78.7{\pm}1.7}$ & ${19.6{\pm}1.7}$ & $98.3{\pm}0.3$ &
${80.3{\pm}5.3}$ & ${14.5{\pm}5.4}$ & $94.9{\pm}1.2$ &
${65.6{\pm}3.2}$ & ${26.4{\pm}3.3}$ & $86.5{\pm}0.8$ \\
\bottomrule
\end{tabular}%
}
\end{table*}

%% file: Tables/6_stacked_layer_tbl.tex
\begin{table*}[t]
\centering
\caption{Comparison of gating (selection) versus stacked fully connected layers (generation) across benchmarks (mean \mbox{$\pm$} std over 10 permutations). Stacked layers are 512→512 with ReLU and dropout.}\looseness=-1
\label{tab:gates_vs_stacked}
\footnotesize
\setlength{\tabcolsep}{2.5pt}
\resizebox{\textwidth}{!}{%
\begin{tabular}{l|ccc|ccc|ccc|ccc}
\toprule
& \multicolumn{3}{c|}{\textbf{PAMAP2}} 
& \multicolumn{3}{c|}{\textbf{DSA}} 
& \multicolumn{3}{c|}{\textbf{UCI-HAR}} 
& \multicolumn{3}{c}{\textbf{RealWorld}} \\
& \multicolumn{3}{c|}{8 subj., 12 act.} 
& \multicolumn{3}{c|}{8 subj., 18 act.} 
& \multicolumn{3}{c|}{30 subj., 6 act.} 
& \multicolumn{3}{c}{15 subj., 8 act.} \\
\cmidrule(lr){2-4} \cmidrule(lr){5-7} \cmidrule(lr){8-10} \cmidrule(lr){11-13}
& FA & FM & LA & FA & FM & LA & FA & FM & LA & FA & FM & LA \\
\midrule
\multicolumn{13}{l}{\textit{Frozen Backbone with Stacked Adaptation Layers}} \\
No stacked layers   & $76.5{\pm}4.0$ & $17.5{\pm}4.0$ & $93.9{\pm}0.4$ &
$75.7{\pm}2.7$ & $22.1{\pm}2.7$ & $97.9{\pm}0.2$ &
$80.1{\pm}9.4$ & $13.9{\pm}9.5$ & $94.0{\pm}1.5$ & $64.4{\pm}3.6$ & $27.3{\pm}3.6$ & $86.2{\pm}0.9$ \\
1 stacked layer  & $74.2{\pm}4.8$ & $20.1{\pm}4.9$ & $94.3{\pm}0.4$ & $74.8{\pm}3.2$ & $23.5{\pm}3.3$ & $98.2{\pm}0.3$ & $78.9{\pm}5.8$ & $16.2{\pm}5.8$ & $95.1{\pm}1.1$ & $62.8{\pm}4.1$ & $29.6{\pm}4.2$ & $87.0{\pm}0.8$ \\
2 stacked layers  & $72.1{\pm}6.2$ & $22.7{\pm}6.3$ & $94.8{\pm}0.4$ & $73.9{\pm}4.3$ & $24.6{\pm}4.5$ & $98.4{\pm}0.3$ & $78.1{\pm}6.1$ & $18.0{\pm}6.1$ & $96.1{\pm}1.1$ & $61.5{\pm}5.2$ & $31.1{\pm}5.3$ & $87.6{\pm}0.8$ \\
3 stacked layers  & $70.1{\pm}7.7$ & $25.3{\pm}7.8$ & $95.3{\pm}0.4$ & $73.0{\pm}5.3$ & $25.6{\pm}5.5$ & $98.6{\pm}0.3$ & $77.3{\pm}6.4$ & $19.7{\pm}6.4$ & $97.0{\pm}1.2$ & $60.2{\pm}6.4$ & $32.9{\pm}6.5$ & $88.1{\pm}0.9$ \\
\midrule
\textbf{Gates (Ours)} &
${77.7{\pm}2.5}$ & ${16.2{\pm}2.6}$ & $93.8{\pm}0.4$ &
${78.7{\pm}1.7}$ & ${19.6{\pm}1.7}$ & $98.3{\pm}0.3$ &
${80.3{\pm}5.3}$ & ${14.5{\pm}5.4}$ & $94.9{\pm}1.2$ &
${65.6{\pm}3.2}$ & ${26.4{\pm}3.3}$ & $86.5{\pm}0.8$ \\
\bottomrule
\end{tabular}%
}
\end{table*}

%% file: Tables/7_KD_tbl.tex
\begin{table*}[t]
\centering
\caption{Effect of knowledge distillation, task-aware gates, and experience replay (mean \mbox{$\pm$} std over permutations).}\looseness=-1
\label{tab:extensions}
\footnotesize
\setlength{\tabcolsep}{2.5pt}
\resizebox{\textwidth}{!}{%
\begin{tabular}{l|ccc|ccc|ccc|ccc}
\toprule
& \multicolumn{3}{c|}{\textbf{PAMAP2}} 
& \multicolumn{3}{c|}{\textbf{DSA}} 
& \multicolumn{3}{c|}{\textbf{UCI-HAR}} 
& \multicolumn{3}{c}{\textbf{RealWorld}} \\
& \multicolumn{3}{c|}{8 subj., 12 act.} 
& \multicolumn{3}{c|}{8 subj., 18 act.} 
& \multicolumn{3}{c|}{30 subj., 6 act.} 
& \multicolumn{3}{c}{15 subj., 8 act.} \\
\cmidrule(lr){2-4} \cmidrule(lr){5-7} \cmidrule(lr){8-10} \cmidrule(lr){11-13}
& FA & FM & LA & FA & FM & LA & FA & FM & LA & FA & FM & LA \\
\midrule
\multicolumn{13}{l}{\textit{Task-Free Gates (Single Shared Gates)}} \\
Without KD (Ours) &
${77.7{\pm}2.5}$ & ${16.2{\pm}2.6}$ & $93.8{\pm}0.4$ &
${78.7{\pm}1.7}$ & ${19.6{\pm}1.7}$ & $98.3{\pm}0.3$ &
${80.3{\pm}5.3}$ & ${14.5{\pm}5.4}$ & $94.9{\pm}1.2$ &
${65.6{\pm}3.2}$ & ${26.4{\pm}3.3}$ & $86.5{\pm}0.8$ \\
With KD           & $77.5{\pm}3.5$ & $12.7{\pm}4.0$ & $90.2{\pm}1.0$ & $78.1{\pm}3.6$ & $17.3{\pm}3.0$ & $95.4{\pm}0.7$ & $82.1{\pm}4.4$ & $8.9{\pm}5.0$ & $90.9{\pm}2.6$ & $65.2{\pm}3.9$ & $22.9{\pm}3.7$ & $83.4{\pm}1.4$ \\
\midrule
\multicolumn{13}{l}{\textit{Task-Aware Gates (Separate Gates per Task)}} \\
Task-Aware        & $78.9{\pm}4.4$ & $12.3{\pm}4.2$ & $91.2{\pm}0.6$ & $80.1{\pm}3.0$ & $16.4{\pm}2.9$ & $96.5{\pm}0.5$ & $82.9{\pm}3.0$ & $8.9{\pm}1.8$ & $91.8{\pm}2.3$ & $66.8{\pm}3.3$ & $21.9{\pm}3.0$ & $84.3{\pm}1.1$ \\
\midrule
\multicolumn{13}{l}{\textit{Gates + Experience Replay (500 sample buffer)}} \\
Gates + Replay    & $84.0{\pm}2.1$ & $6.6{\pm}1.6$ & $90.6{\pm}1.0$ & $85.0{\pm}3.0$ & $10.2{\pm}2.4$ & $95.2{\pm}1.1$ & $83.7{\pm}5.9$ & $7.1{\pm}4.5$ & $90.8{\pm}3.6$ & $71.1{\pm}3.5$ & $16.2{\pm}2.9$ & $83.5{\pm}1.9$ \\
Gates + DER       & ${84.3{\pm}2.0}$ & ${6.1{\pm}1.5}$ & $90.4{\pm}0.9$ & ${85.7{\pm}2.6}$ & ${9.6{\pm}2.1}$ & $95.3{\pm}1.0$ & ${85.0{\pm}5.0}$ & ${6.1{\pm}2.6}$ & $91.1{\pm}4.0$ & ${71.8{\pm}3.1}$ & ${15.3{\pm}2.4}$ & $83.7{\pm}2.0$ \\
\bottomrule
\end{tabular}%
}
\end{table*}

%% file: Tables/8_gate_sweep_tbl.tex
\begin{table*}[t]
\centering
\caption{Effect of the number of channel gates inserted into the frozen backbone (mean \mbox{$\pm$} std over 10 permutations). Gates are added to the first $N$ stacks, from $0$ (frozen classifier only) to $4$ (all stacks, our default). $0$ and $4$ gates correspond to the \textit{Pretrained} and \textit{Pretrained + Gates (Ours)} rows of Table~\ref{tab:freezing_gates}.}\looseness=-1
\label{tab:gate_sweep}
\footnotesize
\setlength{\tabcolsep}{2.5pt}
\resizebox{\textwidth}{!}{%
\begin{tabular}{l|ccc|ccc|ccc|ccc}
\toprule
& \multicolumn{3}{c|}{\textbf{PAMAP2}} 
& \multicolumn{3}{c|}{\textbf{DSA}} 
& \multicolumn{3}{c|}{\textbf{UCI-HAR}} 
& \multicolumn{3}{c}{\textbf{RealWorld}} \\
& \multicolumn{3}{c|}{8 subj., 12 act.} 
& \multicolumn{3}{c|}{8 subj., 18 act.} 
& \multicolumn{3}{c|}{30 subj., 6 act.} 
& \multicolumn{3}{c}{15 subj., 8 act.} \\
\cmidrule(lr){2-4} \cmidrule(lr){5-7} \cmidrule(lr){8-10} \cmidrule(lr){11-13}
& FA $\uparrow$ & FM $\downarrow$ & LA $\uparrow$ & FA $\uparrow$ & FM $\downarrow$ & LA $\uparrow$ & FA $\uparrow$ & FM $\downarrow$ & LA $\uparrow$ & FA $\uparrow$ & FM $\downarrow$ & LA $\uparrow$ \\
\midrule
0 gates (classifier only) &
$76.5{\pm}4.0$ & $17.5{\pm}4.0$ & $93.9{\pm}0.4$ &
$75.7{\pm}2.7$ & $22.1{\pm}2.7$ & $97.9{\pm}0.2$ &
$80.1{\pm}9.4$ & $13.9{\pm}9.5$ & $94.0{\pm}1.5$ &
$64.4{\pm}3.6$ & $27.3{\pm}3.6$ & $86.2{\pm}0.9$ \\
1 gate &
$75.3{\pm}2.4$ & $15.2{\pm}2.7$ & $90.5{\pm}0.8$ &
$70.5{\pm}4.0$ & $25.4{\pm}3.6$ & $95.9{\pm}0.8$ &
$79.5{\pm}3.6$ & $11.7{\pm}4.2$ & $91.2{\pm}2.8$ &
$62.8{\pm}3.8$ & $28.9{\pm}3.7$ & $85.4{\pm}1.0$ \\
2 gates &
$72.6{\pm}3.2$ & $18.2{\pm}3.0$ & $90.8{\pm}0.7$ &
$71.5{\pm}2.7$ & $24.8{\pm}2.2$ & $96.3{\pm}0.8$ &
$78.7{\pm}6.6$ & $14.0{\pm}5.7$ & $92.8{\pm}2.3$ &
$61.9{\pm}4.5$ & $29.6{\pm}4.4$ & $86.0{\pm}1.0$ \\
3 gates &
$77.0{\pm}2.5$ & $13.8{\pm}2.1$ & $90.8{\pm}0.6$ &
$71.9{\pm}6.0$ & $24.0{\pm}5.8$ & $95.9{\pm}0.3$ &
$79.8{\pm}6.0$ & $12.2{\pm}5.8$ & $92.0{\pm}3.0$ &
$64.1{\pm}4.1$ & $27.1{\pm}4.0$ & $86.1{\pm}0.9$ \\
\midrule
\textbf{4 gates (Ours)} &
${77.7{\pm}2.5}$ & ${16.2{\pm}2.6}$ & $93.8{\pm}0.4$ &
${78.7{\pm}1.7}$ & ${19.6{\pm}1.7}$ & $98.3{\pm}0.3$ &
${80.3{\pm}5.3}$ & ${14.5{\pm}5.4}$ & $94.9{\pm}1.2$ &
${65.6{\pm}3.2}$ & ${26.4{\pm}3.3}$ & $86.5{\pm}0.8$ \\
\bottomrule
\end{tabular}%
}
\end{table*}

%% file: Sections/07_conclusion.tex
\section{Conclusion and Future Work}

We presented a parameter-efficient continual learning framework for human activity recognition that combines frozen pretrained backbones with lightweight channel-wise gating mechanisms. Our key insight is that adaptation should operate through feature \emph{selection} rather than feature \emph{generation}: by restricting learned transformations to diagonal scaling of existing features, we preserve the geometric structure of pretrained representations while enabling subject-specific modulation. This bounded form of adaptation limits representational drift and reduces catastrophic forgetting without requiring replay buffers or task-specific regularization. These properties make our approach well-suited for on-device deployment in IoT applications such as wearable health monitoring, where privacy constraints preclude data transmission and memory limitations prohibit replay buffers.\looseness=-1

Figure~\ref{fig:forgetting_conclusion} illustrates the effectiveness of our approach on the same PAMAP2 subject sequence shown in Figure~\ref{fig:forgetting_intro}, demonstrating substantially reduced forgetting. Extensive experiments across foure HAR benchmarks further validated our approach. On PAMAP2, freezing the backbone reduced forgetting from 39.7\% to 17.5\%, while adding gates further improved final accuracy to 77.7\% with only 16.2\% forgetting, outperforming regularization-based methods (EWC, LwF) and architecture-based methods (HAT) while training less than 2\% of model parameters. We also demonstrated that gating is complementary to replay: combining gates with DER achieved 84.3\% final accuracy and 6.1\% forgetting, competitive with full-model replay methods at a fraction of the parameter cost.\looseness=-1

\begin{figure}
    \centering
    \includegraphics[width=\linewidth]{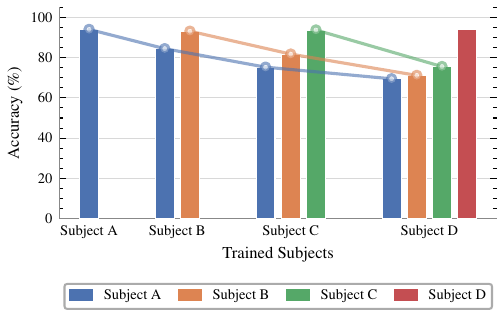}
    \caption{Per-subject accuracy evolution with our gating approach on PAMAP2. Unlike the baseline in Figure~\ref{fig:forgetting_intro}, which dropped to 40\% on Subject~1 after four subjects, our method retains 69.5\% accuracy while achieving 93.8\% learning accuracy on new subjects.\looseness=-1}
    \label{fig:forgetting_conclusion}
\end{figure}

Our approach has limitations that suggest directions for future work. The reliance on a pretrained backbone assumes access to a sufficiently diverse source dataset; performance may degrade when source and target domains differ substantially. Additionally, while task-free gates perform well, explicitly modeling task structure could further reduce forgetting in scenarios where task boundaries are available. Future work could explore adaptive gating architectures that automatically determine the optimal number and placement of gates, as well as extensions to other domain-incremental settings beyond HAR.\looseness=-1